\documentclass{article}

\usepackage{microtype}
\usepackage{amsthm,amssymb,amsmath}
\usepackage{mathtools}
\usepackage{bbm}
\usepackage{amsfonts}
\usepackage{algorithm}
\usepackage[noend]{algorithmic}
\usepackage{thm-restate}
\usepackage{cuted}

\usepackage[pagebackref]{hyperref}
\renewcommand*{\backrefalt}[4]{%
    \ifcase #1 \footnotesize{(Not cited.)}%
    \or        \footnotesize{(Cited on page~#2)}%
    \else      \footnotesize{(Cited on pages~#2)}%
    \fi}

\theoremstyle{plain}
\newtheorem{theorem}{Theorem}[section]
\newtheorem{lemma}[theorem]{Lemma}

\theoremstyle{definition}

\newtheorem{assumption}[theorem]{Assumption}
\theoremstyle{remark}

\DeclarePairedDelimiter{\abs}{\lvert}{\rvert} %
\DeclarePairedDelimiter{\brk}{[}{]}
\DeclarePairedDelimiter{\crl}{\{}{\}}
\DeclarePairedDelimiter{\prn}{(}{)}
\DeclarePairedDelimiter{\nrm}{\|}{\|}

\newcommand{\inner}[2]{\left\langle #1,\, #2 \right\rangle}

\DeclareMathOperator*{\argmin}{arg\,min}

\newcommand{\collapse}[1]{$\dots$}

\newcommand{\T}{^\top}

\newcommand{\mc}[1]{\mathcal{#1}}

\newcommand{\R}{{\mathbb{R}}}

\newcommand{\E}{{\mathbb E}}

\usepackage[accepted]{icml2023_modified}

\newcommand{\PP}[0]{\mathbf{P}}

\newcommand{\w}[0]{\mathbf{w}}
\newcommand{\f}[0]{L}
\newcommand{\g}[0]{\mathbf{g}}
\newcommand{\F}[0]{\hat{L}}
\newcommand{\h}[0]{h}
\newcommand{\uu}[0]{\mathbf{u}}
\newcommand{\vv}[0]{\mathbf{v}}
\renewcommand{\phi}{\varphi}

\newcommand{\bT}[0]{\mathsf{Telescope}}
\newcommand{\bN}[0]{\mathsf{Noise}}
\newcommand{\bI}[0]{\mathsf{Inexact}}
\newcommand{\bC}[0]{\mathsf{Cancel}}

\makeatletter
\newcommand{\vast}{\bBigg@{4}}
\makeatother

\icmltitlerunning{Two Losses Are Better Than One: Faster Optimization Using a Cheaper Proxy}
\begin{document}
\twocolumn[
\icmltitle{Two Losses Are Better Than One: Faster Optimization Using a Cheaper Proxy}


\icmlsetsymbol{equal}{*}

\begin{icmlauthorlist}
\icmlauthor{Blake Woodworth}{yyy}
\icmlauthor{Konstantin Mishchenko}{comp}
\icmlauthor{Francis Bach}{yyy}
\end{icmlauthorlist}

\icmlaffiliation{yyy}{Inria, Ecole Normale Sup{\'e}rieure, PSL Research University, Paris France}
\icmlaffiliation{comp}{Samsung AI Center, Cambridge, UK. Work done while at CNRS, Ecole Normale Sup{\'e}rieure, Inria}

\icmlcorrespondingauthor{Blake Woodworth}{blakewoodworth@gmail.com}
\icmlcorrespondingauthor{Konstantin Mishchenko}{konsta.mish@gmail.com}
\icmlcorrespondingauthor{Francis Bach}{francis.bach@inria.fr}

\icmlkeywords{}

\vskip 0.3in
]
\printAffiliationsAndNotice{}  

\begin{abstract}
We present an algorithm for minimizing an objective with hard-to-compute gradients by using a related, easier-to-access function as a proxy. Our algorithm is based on approximate proximal-point iterations on the proxy combined with relatively few stochastic gradients from the objective. When the difference between the objective and the proxy is $\delta$-smooth, our algorithm guarantees convergence at a rate matching stochastic gradient descent on a $\delta$-smooth objective, which can lead to substantially better sample efficiency. Our algorithm has many potential applications in machine learning, and provides a principled means of leveraging synthetic data, physics simulators, mixed public and private data, and more.
\end{abstract}

\section{Introduction}

In many machine learning problems, it is difficult to access the objective function, e.g., to compute its gradients for use in training. This difficulty can come from a wide range of sources. For example, it might be costly or time consuming to collect labelled training examples, leading to a small training set. In some applications, computing the gradient of the loss might require waiting for an actual robot to execute a policy in a physical environment, which could be very slow. Moreover, physical execution can risk damage to the robot or its surroundings. Finally, individuals' data may be subject to privacy or legal constraints, limiting access for  gradient computations. In all of these cases, we face challenges that complicate training, and can lead to poor performance using off-the-shelf algorithms like stochastic gradient descent (\textsc{SGD}).

A natural solution to the problems above is to find a second, easier-to-access ``proxy'' function that can be used in place of the real objective. When few samples are available, synthetic data could be used to approximate the objective, or in applications with sensitive data, we could train a generative model using a privacy-preserving algorithm and define the proxy using synthetic samples. To mitigate the cost and risk of operating a physical robot, one could use a simulator that is faster, cheaper, and safer.  

Unfortunately, directly optimizing a hand-crafted proxy objective might not improve performance on the original problem, because the surrogate loss can introduce a bias. We aim to remedy this in our work, and our main contribution is a general algorithm that exploits curvature information from the proxy loss in order to better optimize the objective that we care about.

\subsection{The problem}
Our ultimate aim is to solve unconstrained convex optimization problems of the form
\begin{equation}\label{eq:the-objective}
\min\nolimits_{\w \in \R^d}\f(\w)\,.
\end{equation}
Here, $\f$ could represent the population loss of a machine learning model parametrized by weights $\w$; it could also be the training loss, or any other objective function to be minimized. To solve \eqref{eq:the-objective}, we use a proxy function $\F$ and stochastic gradients $\g_k$ approximating $\nabla \f(\w_k)$ and update
\begin{align}\label{eq:update}
\w_{k+1} \approx \argmin_{\w} \Bigl\{\underbrace{\hat L(\w)}_{\textrm{Proxy loss}} +  &\underbrace{\langle\g_k - \nabla \hat L(\w_k), \w - \w_k \rangle}_{\textrm{Bias correction}} \notag \\
&\quad + \underbrace{\frac{1}{2\eta}\nrm{\w - \w_k}^2}_{\textrm{Regularization}}\Bigr\}.
\end{align}
Thus, each iteration uses just a single stochastic gradient from $\f$, but requires a heavier computation based on $\F$. If we take $\F \equiv 0$ 
, we note that the update \eqref{eq:update} is equivalent to one stochastic gradient descent step
\begin{equation}\label{eq:sgd-update}
    \w_{k+1} = \w_k - \eta \g_k\,.
\end{equation}
Therefore, \textsc{SGD} is a natural point of comparison for our algorithm, and the main question is to what extent updating using \eqref{eq:update} (with non-zero $\F$) rather than \eqref{eq:sgd-update} is worth the additional computational cost. We will show that when $\F$ is sufficiently ``similar'' to $\f$, our algorithm's updates, \eqref{eq:update}, can converge to the minimum of $\f$ in substantially fewer iterations than would be needed for \textsc{SGD}, \eqref{eq:sgd-update}. Accordingly, the number of stochastic gradients from $\f$ that our algorithm needs is less than what would be needed by \textsc{SGD}.


For this to work, we obviously require that $\F$ somehow resembles $\f$, and the particular notion of similarity that we consider is that the function $h := \f - \F$ is differentiable and has $\delta$-Lipschitz continuous gradients, which, in the case that $\f$ and $\F$ are twice-differentiable, is equivalent to requiring that $\forall_{\w}\ \nrm{\nabla^2 \f(\w) - \nabla^2 \F(\w)}_{\textrm{op}} \leq \delta$. This measure, sometimes referred to as ``$\delta$-Hessian similarity'', is common in the optimization literature \citep[see, e.g.,][]{mairal2013optimization, mairal2015incremental,arjevani2015communication,kovalev2022optimal,chayti2022optimization}. For certain problems, it can be easy to construct a proxy that satisfies this similarity condition for small~$\delta$. For instance, the Hessian of a least-squares objective is simply the feature covariance matrix, so an appropriate proxy can be defined using any estimate of the covariance matrix that is $\delta$-accurate in the operator norm. In Section \ref{sec:binary-logistic-regression} we further show that for least squares and logistic regression, proxies with $\delta=0$ can be constructed \emph{without using any labels}.

\subsection{Our contributions}
We present Algorithm \ref{alg:ISMD}, which substitutes cheaper accesses to a proxy loss in place of costly accesses to the objective itself. In Section \ref{sec:convergence-analysis}, we prove convergence guarantees for our algorithm in the convex and strongly convex settings, and in Section \ref{sec:comparison-with-direct-minimization}, we show that these guarantees imply statistical optimality as well as a dependence on the similarity parameter, $\delta$, in the place of other properties of~$\f$ itself, such as its potentially much larger smoothness parameter. On a technical level, our analysis improves over many similar methods by allowing each iteration's proximal-point subproblem to be solved inexactly, with the inexactness captured by a simple criterion that can be evaluated at the time of execution. Furthermore, in Section \ref{sec:making-the-gradient-small} we show that this criterion can be satisfied efficiently owing to the strong convexity of the proximal-point subproblem's objective. In Section \ref{sec:applications}, we discuss potential applications of our algorithm where the target objective is costly to access and a suitable proxy is available. In Section \ref{sec:experiments}, we conduct experiments to show the efficacy of our algorithm in realistic problems, including one with a non-convex objective. Finally, in Section \ref{sec:conclusion}, we discuss several possible extensions of our work and we provide a preliminary analysis of our algorithm for the non-convex setting.

\subsection{Background and related work}\label{sec:related-work}

Using a more tractable surrogate in the place of the actual target loss is a very old idea in statistics and machine learning, and some type of surrogate is used in almost any application. For instance, convex loss functions like the hinge or logistic loss are often used as surrogates for the discontinuous 0-1 classification loss. This avoids the computational intractability of minimizing the 0-1 loss \citep[see, e.g.,][]{arora1997hardness}, but note that this does not per se mitigate other difficulties with accessing the objective in the sense that we have discussed, e.g., limited access to training samples. 
Another classic example is empirical risk minimization \citep{vapnik1968uniform} where the training loss is used as a surrogate for the real objective, the population loss. This more closely aligns with our motivation  since it  replaces the hard-to-evaluate population loss---which is defined by an unknown data distribution---with the training loss, which we can directly evaluate, compute gradients of, etc. 

However, both of these examples rely on i.i.d.~samples from the target distribution while our approach can exploit additional ``side-information''---synthetic data, simulators, etc.---to complement a smaller collection of i.i.d.~samples. Moreover, the way in which surrogates are classically used differs from our approach: typically, the surrogate is directly minimized and the result is taken as an estimate of the minimum of the target objective, whereas we use the proxy to facilitate optimization using stochastic gradients from the objective itself. As such, the classic approach relies upon minima of the surrogate approximating minima of the target, while our method only requires the proxy and objective to have similar second derivatives. This is not \emph{necessarily} easier to achieve, but it is easier in certain cases (see, e.g., Section \ref{sec:binary-logistic-regression}), and it is much more amenable to applications involving synthetic data, simulators, etc., where the proxy loss could have different minima due to idiosyncrasies of the synthetic data or simulator.

In more closely related work, \citet{hendrikx2020statistically} propose an algorithm, SPAG, for a strongly convex setting, which builds upon a previous method, DANE \citep{shamir2014communication}, and which resembles an accelerated version of Algorithm \ref{alg:ISMD}. Their guarantees are analogous to ours up to acceleration,  however both SPAG and DANE require exact proximal-point updates, while our algorithm and analysis allow these updates to be inexact. Motivated by the problem of minimizing over a subset of parameters, \citet{parpas2017multilevel} assumed that the surrogate subproblem is defined over a different space, with projection into that space and back only introducing a multiplicative error. In other related work, \citet{chayti2022optimization} propose several algorithms that use a surrogate for non-convex optimization, one of which resembles a momentum variant of Algorithm \ref{alg:ISMD}. 

\citet{mairal2013optimization, mairal2015incremental} studies the ``majorization-minimization'' meta-algorithm, where an upper bound on the objective is minimized in each iteration. Our method can be viewed as an instance of majorization-minimization with the surrogate being used to define the upper bound on the loss. \citeauthor{mairal2013optimization} analyzes this method under the same $\delta$-similarity condition that we use and prove similar convergence rates, but the key difference between this work and ours is that \citeauthor{mairal2013optimization}'s methods require a guaranteed upper bound on the objective---or on components of the objective in the case of finite sum structured problems---and require that this upper bound be exactly minimized in each iteration, whereas we only require an in-expectation upper bound and allow for only approximate minimization. In this way, our setting is more appropriate for the stochastic optimization problems that arise in machine learning.

Finally, in the most closely related work, \citet{mairal2013stochastic} study a stochastic majorization-minimization approach which is applicable to the stochastic optimization problems we are interested in. However, \citeauthor{mairal2013stochastic}'s approach requires a guaranteed upper bound of the loss function evaluated on each sample and requires that the $\delta$-similarity bound hold almost surely for each sample. In contrast, we only require an in-expectation upper bound and that the $\delta$-similarity hold in expectation, making it easier to find a suitable surrogate. Furthermore, our convergence guarantees have a better scaling with the number of iterations, $K$.

\section{The algorithm and its analysis}
We now present our algorithm and prove convergence guarantees in the convex and strongly convex settings. We will also argue that our method can provide substantially better guarantees than \textsc{SGD} when $\F$ is sufficiently similar to $\f$.

\begin{algorithm}[t]
\caption{\textsc{ProxyProx} \label{alg:ISMD}}
\begin{algorithmic}[1]
\STATE \textbf{Input:} initialization $\w_0\in\R^d$, stepsize $\eta>0$
\FOR{$k=0,1,\dots,K-1$} 
\STATE Sample $\g_k$ such that $\E[\g_k\mid \w_k]=\nabla L(\w_k)$
\STATE Set $\w_{k+1} \approx \argmin_\w \phi_k(\w)$ where
\[
\phi_k(\w) := \inner{\g_k}{\w} + D_{\F}(\w;\w_k) + \smash{\frac{1}{2\eta}}\nrm{\w-\w_k}^2.
\]
\ENDFOR
\end{algorithmic}
\end{algorithm}

\subsection{Setting and notation}\label{sec:setting-and-notation}
Throughout this paper, we will use $\nrm{\uu}$ to denote the Euclidean norm of the vector $\uu$. We will assume the minimum value of $\f$, which we denote $\f^*$, is realized, with $\w^*$ denoting an arbitrary minimizer. We use $B^2 := \E\nrm{\w_0 - \w^*}^2$ as shorthand to denote the distance between the (possibly random) initialization, $\w_0$, and this minimizer.

A function $f$ is $\mu$-strongly convex if for all $\uu,\vv$
\begin{equation}\label{eq:def-strong-convexity}
f(\vv) \geq f(\uu) + \inner{\nabla f(\uu)}{\vv-\uu} + \frac{\mu}{2}\nrm{\uu - \vv}^2,
\end{equation}
where, in the event that $f$ is not differentiable, $\nabla f(\uu)$ denotes here an arbitrary subgradient of $f$. When this holds for $\mu = 0$, $f$ is merely convex. A function $f$ is $H$-smooth if it is differentiable and $\nabla f$ is $H$-Lipschitz continuous or, equivalently, for all $\uu,\vv$
\begin{equation}\label{eq:smoothness-inequality}
\abs*{f(\vv) - f(\uu) - \inner{\nabla f(\uu)}{\vv-\uu}} 
\leq \frac{H}{2}\nrm{\uu - \vv}^2.
\end{equation}
We define the Bregman divergence with potential $\psi$ as
\begin{equation}\label{eq:def-bregman-divergence}
D_{\psi}(\uu;\vv) := \psi(\uu) - \psi(\vv) - \inner{\nabla \psi(\vv)}{\uu - \vv}.
\end{equation}
While Bregman divergences are typically defined only for strictly convex potentials $\psi$, in this paper we allow $\psi$ to be any differentiable function with $D_\psi$ simply defined as in \eqref{eq:def-bregman-divergence}. The key property of Bregman divergences for our purposes is the three-point identity:\footnote{This does not require convexity \citep[][Lemma 3.1]{chen1993convergence}.} for all $\uu,\vv,\w$
\begin{multline}\label{eq:def-three-point-identity}
D_\psi(\uu;\vv) - D_{\psi}(\uu;\w) - D_{\psi}(\w;\vv) \\
= \inner{\nabla \psi(\vv) - \nabla \psi(\w)}{\w-\uu}.
\end{multline}

Our algorithm relies upon two main assumptions; one which captures the similarity between the objective $\f$ and the proxy $\F$, and one which captures the level of noise in the stochastic gradients $\g_k$. First, we assume
\begin{assumption}\label{ass:similarity}
The function $\h := \f - \F$ is differentiable and $\nabla \h$ is $\delta$-Lipschitz continuous.
\end{assumption}
That is, we require that our proxy $\F$ has similar curvature to the objective $\f$. In light of \eqref{eq:smoothness-inequality}, Assumption \ref{ass:similarity} implies that 
\begin{equation*}
\abs*{D_{\h}(\uu;\vv)} = \abs*{D_{\f}(\uu;\vv) - D_{\F}(\uu;\vv)} \leq \frac{\delta}{2}\nrm{\uu - \vv}^2.
\end{equation*}
This justifies the definition of $\phi_k$ in Algorithm \ref{alg:ISMD} because, when $\delta$ is small
\begin{equation*}
\phi_k(\w) \approx \inner{\g_k}{\w} + D_{\f}(\w;\w_k) + \smash{\frac{1}{2\eta}}\nrm{\w-\w_k}^2,
\end{equation*}
which corresponds to (inexact, stochastic) proximal-point iterations directly on the objective $\f$, which is well-known to have good performance \citep[see, e.g.,][]{barre2022principled}.

We note that Assumption \ref{ass:similarity} can accommodate non-differentiable objectives, e.g., if $\f(\w) = \ell(\w) + \psi(\w)$ and $\F(\w) = \hat{\ell}(\w) + \psi(\w)$ with $\ell$ and $\hat{\ell}$ differentiable, because $\psi$ does not appear in $\f - \F$. In this way, our algorithm is compatible, e.g., with non-smooth regularizers such as an $L_1$ penalty as long as this is incorporated into the proxy.


In addition to Assumption \ref{ass:similarity}, we require the following:
\begin{assumption}\label{ass:noise}
For each $k$, $\E\brk*{\g_k\,\middle|\,\w_k} = \nabla \f(\w_k)$ and $\E\brk*{\nrm{\g_k - \nabla \f(\w_k)}^2\,\middle|\,\w_k} \leq \sigma^2$. If $\f$ is not differentiable but convex, this holds for some subgradient in $\partial\f(\w_k)$.
\end{assumption}
The assumption that the stochastic gradients $\g_k$ are unbiased and have bounded variance is very common in the optimization literature \citep{nemirovskyyudin1983,dekel2012optimal,bubeck2015convex}.


\subsection{Convergence guarantees and proof sketch}\label{sec:convergence-analysis}

We begin with our main result:
\begin{restatable}{theorem}{stronglyconvexthm}\label{thm:strongly-convex}
Under Assumptions \ref{ass:similarity} and \ref{ass:noise}, let $\f$ be $\mu$-strongly convex and let $\eta \leq \frac{1}{4\delta}$. If for each $k$
\[
\E\nrm{\nabla \phi_k(\w_{k+1})}^2 
\leq
\frac{\mu}{4\eta}\E\nrm{\w_{k+1} - \w_k}^2 + G^2,
\]
for some $G$, then a geometrically weighted average of the iterates of Algorithm \ref{alg:ISMD}, $\bar{\w}_K$ (see \eqref{eq:weighted-iterate}), satisfies
\[
\E \f(\bar{\w}_K) - \f^* \leq  \frac{5B^2}{8\eta} \Bigl( 1 + \frac{2\eta\mu}{5}\Bigr)^{1-K}
+ 2\eta\sigma^2 + \frac{G^2}{\mu},
\]
where the expectation is taken over the randomness in both the stochastic gradients and the selection of the iterates. 
\end{restatable}
A detailed proof can be found in Appendices \ref{app:proof-of-thmstronglyconvex} and \ref{app:proof-of-thmconvex}. 
We start with a straightforward algebraic manipulation of $\nabla \phi_k$. Recalling that $\h = \f - \F$, we write\footnote{Here, for simplicity, assume $\f$ and $\F$ are differentiable, but see Appendix \ref{app:proof-of-thmstronglyconvex} for the general case.}
\begin{multline}\label{eq:nabla-phi_k}
\nabla \phi_k(\w_{k+1}) 
= -\nabla \h(\w_{k+1}) + \nabla \h(\w_k) + \nabla \f(\w_{k+1}) \\ + \g_k - \nabla \f(\w_k) + \frac{\w_{k+1}-\w_k}{\eta}.
\end{multline}
Next, \eqref{eq:def-bregman-divergence} implies
\begin{multline}\label{eq:bregman-identity-proof-sketch}
\f(\w_{k+1}) - \f^* \\
= \inner{\nabla \f(\w_{k+1})}{\w_{k+1} - \w^*} - D_{\f}(\w^*;\w_{k+1}).
\end{multline}
Rearranging \eqref{eq:nabla-phi_k} allows us to express $\nabla \f(\w_{k+1})$ in terms of $\nabla \phi_k(\w_{k+1})$ and the other quantities, which we substitute in the place of $\nabla \f(\w_{k+1})$ in \eqref{eq:bregman-identity-proof-sketch}. After several straightforward algebraic manipulations, including using the three-point identity \eqref{eq:def-three-point-identity}, we arrive at the key identity
\begin{align}\label{eq:proof-outline-identity}
&\f(\w_{k+1}) - \f^* \\
&\begin{aligned}
&= \frac{1}{2\eta}\nrm*{\w_k - \w^*}^2 - D_{\h}(\w^*;\w_k) \\
&- \frac{1}{2\eta}\nrm*{\w_{k+1} - \w^*}^2 + D_{\h}(\w^*;\w_{k+1}) 
\end{aligned}
\tag*{{\vast\}} $(\bT)$\ } \nonumber\\
&+ \inner{\nabla \f(\w_k) - \g_k}{\w_{k+1} - \w^*} \tag*{{\big\}}\;\;\; $(\bN)$\quad\;} \nonumber\\
&+ \inner{\nabla \phi_k(\w_{k+1})}{\w_{k+1} - \w^*} \tag*{{\big\}}\; \ $(\bI)$\;\; } \nonumber\\
&\begin{aligned}
&+ D_{\h}(\w_{k+1};\w_k) - \frac{1}{2\eta}\nrm*{\w_{k+1} - \w_k}^2  \\
& - D_{\f}(\w^*;\w_{k+1})\,.
\end{aligned} \tag*{{\Bigg\}}\; $(\bC)$\quad} \nonumber
\end{align}
Although the expression looks unwieldy at first, it leads us directly towards a strategy for analyzing Algorithm \ref{alg:ISMD}.

The $\smash{\bT}$ terms resemble those typically found in the analysis of convex optimization algorithms. In the course of proving Theorem \ref{thm:strongly-convex}, we, roughly speaking, average \eqref{eq:proof-outline-identity} over $k$, causing all but two of these terms to cancel out, and we show that under Assumption \ref{ass:similarity} with $\eta \leq \frac{1}{4\delta}$, their total contribution to the error is small. 

The $\smash{\bN}$ term captures the effect of using noisy gradients~$\g_k$ rather than $\nabla \f(\w_k)$. Under Assumption \ref{ass:noise}, its expectation is at most $2\eta\sigma^2 + \frac{1}{4\eta}\E\nrm{\w_{k+1} - \w_k}^2$. Importantly, since $\g_k$ is unbiased, we can bound this in terms of $\nrm{\w_{k+1} - \w_k}$ rather than $\nrm{\w_{k+1} - \w^*}$.

The $\smash{\bI}$ term captures the effect of only approximately computing the proximal-point update. If we could just set $\w_{k+1} = \argmin_{\w} \phi_k(\w)$, the $\smash{\bI}$ term would be zero, but this is impossible in practice. Instead, a simple application of Young's inequality bounds this term by $\frac{1}{\mu}\nrm{\nabla \phi_k(\w_{k+1})}^2 + \frac{\mu}{4}\nrm{\w_{k+1} - \w^*}^2$.

Nowhere above did we use any property of $\f$ besides those implied by Assumptions \ref{ass:similarity} and \ref{ass:noise}, but to address the $\smash{\bC}$ terms, which we use to eliminate unwanted quantities introduced above, we finally use that in the $\mu$-strongly convex setting, $D_{\f}(\w^*;\w_{k+1}) \geq \frac{\mu}{2}\nrm{\w_{k+1} - \w^*}^2$. So, when $\w_{k+1}$ is chosen so that $\nrm{\nabla \phi_k(\w_{k+1})}^2$ is sufficiently small, this allows us to cancel out everything in $\smash{\bN}$ and $\smash{\bI}$ except for $2\eta\sigma^2$, and this leaves an extra $-\frac{\mu}{4}\nrm{\w_{k+1} - \w^*}^2$ to be incorporated into the $\smash{\bT}$ terms to generate geometric shrinkage. Here, the strong convexity of $\f$ is crucially needed to cancel the dependence of the $\smash{\bI}$ term on $\nrm{\w_{k+1} - \w^*}$---without strong convexity, it is not immediately clear how the $\smash{\bI}$ term can be controlled without setting $\nabla \phi_k(\w_{k+1})=0$. Filling in the details and using standard arguments completes the proof.

While Theorem \ref{thm:strongly-convex} requires $\f$ to be $\mu$-strongly convex, a similar guarantee can be achieved when $\f$ is convex via regularization:
\begin{restatable}{theorem}{thmconvex}\label{thm:convex}
Under Assumptions \ref{ass:similarity} and \ref{ass:noise}, let $\f$ be convex and let $\eta \leq \frac{1}{4\delta}$. Then applying Algorithm \ref{alg:ISMD} to the regularized objectives $\f^{(\mu)}(\w) := \f(\w) + \frac{\mu}{2}\nrm{\w - \w_0}^2$ and $\F^{(\mu)}(\w) := \F(\w) + \frac{\mu}{2}\nrm{\w - \w_0}^2$ with $\mu = \frac{1}{\eta K}$ and any $\eta \leq \frac{1}{4\delta}$ ensures that if for each $k$
\[
\E\nrm{\nabla \phi_k(\w_{k+1})}^2 \leq
\frac{1}{4\eta^2 K}\E\nrm{\w_{k+1} - \w_k}^2 + G^2
\]
for some $G$, then
\[
\E\f\prn*{\frac{1}{K}\smash{\sum_{k=1}^K} \w_k} - \f^*
\leq \frac{9B^2}{8\eta K} + 2\eta\sigma^2 + \eta K G^2.
\]
\end{restatable}
The proof, which we defer to Appendix \ref{app:proof-of-thmconvex}, simply observes that $\f^{(\mu)}$ and $\F^{(\mu)}$ satisfy the conditions needed for Theorem \ref{thm:strongly-convex}, with larger $\mu$ leading to faster convergence. At the same time, when $\mu$ is small enough, optimizing $\f^{(\mu)}$ is essentially equivalent to optimizing $\f$ itself. We conclude by choosing $\mu$ to balance between these competing goals.

Finally, choosing $\eta$ optimally in the strongly convex and convex settings yields the following complexity guarantee, which we prove in Appendix \ref{app:corollary}:
\begin{restatable}{corollary}{stepsizecorollary}\label{cor:optimized-guarantees}
In the setting of Theorem \ref{thm:strongly-convex}, there is a universal constant $c$ s.t.~for any $\varepsilon > 0$, when $G^2 \leq \frac{1}{2}\mu\varepsilon$ and
\[
\eta = \frac{5}{\mu (K-1)}\Bigl(1 + \ln\Bigl(\frac{5B^2\mu (K-1)}{\varepsilon}\Bigr)\Bigr),
\]
then the output, $\hat{\w}$, of Algorithm \ref{alg:ISMD} will have suboptimality at most $\E\f(\hat{\w}) - \f^* \leq \varepsilon$ whenever
\[
K \geq c\cdot\Bigl(1 + \frac{\delta}{\mu} + \frac{\sigma^2}{\mu\varepsilon}\Bigr) 
\ln\Bigl(e + \frac{(\mu+\delta)B^2}{\varepsilon} + \frac{\sigma^2 B^2}{\varepsilon^2}\Bigr).
\]
Likewise, in the setting of Theorem \ref{thm:convex}, there is a universal constant $c$ s.t.~for any $\varepsilon > 0$, when $G^2 \leq \frac{\sigma^2}{K} + \frac{\delta^2 B^2}{K^2}$ and
\[
\eta = \min \Bigl\{\frac{1}{4\delta},\,\frac{B}{\sigma \sqrt{K}}\Bigr\},
\]
then the output, $\hat{\w}$, of Algorithm \ref{alg:ISMD} will have suboptimality at most $\E\f(\hat{\w}) - \f^* \leq \varepsilon$ whenever
\[
K \geq c\cdot\Bigl(\frac{\delta B^2}{\varepsilon} + \frac{\sigma^2 B^2}{\varepsilon^2}\Bigr).
\]
\end{restatable}

\subsection{Making $\nabla \phi_k$ small}\label{sec:making-the-gradient-small}

Theorems \ref{thm:strongly-convex} and \ref{thm:convex} require that each update satisfy
\begin{equation*}
\E\nrm{\nabla \phi_k(\w_{k+1})}^2 \leq
\frac{\mu}{4\eta}\E\nrm{\w_{k+1} - \w_k}^2 + G^2,
\end{equation*}
with $\mu = \frac{1}{\eta K}$ in the convex setting. This is a convenient inexactness criterion, because it consists entirely of quantities that can be evaluated at the time of execution, and similar bounds on the inaccuracy have been studied in the context of the classic proximal-point algorithm~\citep[see, e.g.,][]{rockafellar1976augmented,barre2022principled}.  Nevertheless, this raises the question of how difficult it is to find such a $\w_{k+1}$. It turns out that this can be quite easy due to the $\frac{1}{\eta}$-strong convexity of the subproblem objective $\phi_k$. 
For example, we could use stochastic gradient descent:
\begin{restatable}{proposition}{makinggradientsmall}\label{prop:making-gradient-small}
Let $\F$ be convex and $H$-smooth and let stochastic gradients be available for $\F$ with variance at most $\rho^2$. Then for constants $c,c'$, the output, $\hat{\w}$, of 
\[
T \geq c(1+H\eta)\max\Bigl\{\ln\Bigl(\frac{c'\prn*{1+H\eta}^2}{\mu\eta}\Bigr), 
\frac{\rho^2}{G^2} \Bigr\}
\]
steps of \textsc{SGD} on $\phi_k$ initialized at $\w_k$ satisfies
\[
\E \nrm{\nabla \phi_k(\hat{\w})}^2 \leq \smash{\frac{\mu}{4\eta}}\nrm{\hat{\w} - \w_k}^2 + G^2.
\]
\end{restatable}
The proof, which we defer to Appendix \ref{app:proof-of-makinggradientsmall}, is based on the observation that $\F$ being $H$-smooth implies that $\phi_k$ is $(H+\frac{1}{\eta})$-smooth, so $\nrm{\nabla \phi_k(\hat{\w})}^2 \leq 2(H+\frac{1}{\eta})(\phi_k(\hat{\w}) - \phi_k^*)$, and plugging in the standard suboptimality guarantee for $T$ steps of \textsc{SGD} on a smooth and strongly convex objective eventually leads to the desired bound. 

There is nothing particularly special about using \textsc{SGD} to compute the update $\w_{k+1}$. An advantage of our method is that any algorithm for approximately minimizing $\phi_k$ will suffice when $\F$ is smooth, so we can exploit structure in~$\phi_k$ by using a specialized optimization algorithm. There are also other, more exotic \textsc{SGD} variants that directly target a solution with a small gradient norm \citep{allen2018make,foster2019complexity}, and can have a better dependence on $\rho^2$.

\subsection{Comparison against directly optimizing $\f$}\label{sec:comparison-with-direct-minimization}

To get a better sense of what Theorems \ref{thm:strongly-convex} and \ref{thm:convex} mean, we will compare them against what we could do by just directly optimizing the objective $\f$ using stochastic gradients. Suppose that Assumption \ref{ass:noise} holds, the objective $\f$ is $H$-smooth, $\nrm{\w_0 - \w^*}^2 = B^2$, and $\f$ is either convex or $\mu$-strongly convex. In this case, a natural and popular approach to minimizing $\f$ using a total of $K$ stochastic gradients would be to execute $K$ steps of \textsc{SGD}. Ignoring constant factors, it is well-known that this would yield an $\varepsilon$-accurate solution when \citep{nemirovskyyudin1983}
\begin{align}
K &\geq \frac{HB^2}{\varepsilon} + \frac{\sigma^2 B^2}{\varepsilon^2} \label{eq:comparison-to-sgd-convex}\\
K &\geq \frac{H}{\mu}\ln\frac{HB^2}{\varepsilon} + \frac{\sigma^2}{\mu \varepsilon}.\label{eq:comparison-to-sgd-strongly-convex}
\end{align}
Instead, given a proxy $\F$ satisfying Assumption \ref{ass:similarity}, we could use the same number of stochastic gradients from~$\f$ to execute $K$ steps of Algorithm \ref{alg:ISMD}. Ignoring constants and a logarithmic factor in the strongly convex setting, Corollary \ref{cor:optimized-guarantees} shows that this yields an $\varepsilon$-accurate solution when
\begin{align}
K
&\geq \frac{\delta B^2}{\varepsilon} + \frac{\sigma^2 B^2}{\varepsilon^2} \label{eq:ismd-convex}\\
K 
&\geq \frac{\delta}{\mu}\ln\frac{\delta B^2}{\varepsilon} + \frac{\sigma^2}{\mu \varepsilon}.\label{eq:ismd-strongly-convex}
\end{align}
Comparing \eqref{eq:comparison-to-sgd-convex} and \eqref{eq:comparison-to-sgd-strongly-convex} against \eqref{eq:ismd-convex} and \eqref{eq:ismd-strongly-convex}, we see that Algorithm \ref{alg:ISMD} essentially replaces \textsc{SGD}'s dependence on $H$, the smoothness constant of $\f$, with $\delta$, the smoothness constant of $\h=\f-\F$. So if a proxy $\F$ can be found with $\delta \ll H$, Algorithm \ref{alg:ISMD} may have a much better guarantee than \textsc{SGD}. 

On the other hand, the second, ``statistical terms'' of our guarantees match those of \textsc{SGD}. These statistical terms capture information-theoretic limits to optimizing $\f$ using only $K$ stochastic gradients with variance $\sigma^2$, and these are known to be unimprovable in the worst case \citep{nemirovskyyudin1983}, so our algorithm is statistically optimal (up to a logarithmic factor in the strongly convex setting). In fact, long-standing lower bounds show that obtaining an $\varepsilon$-accurate solution using \emph{any} algorithm that only uses $K$ stochastic gradients from $\f$ requires at least \citep{nemirovskyyudin1983}
\begin{align*}
K &\geq \frac{HB^2}{\sqrt{\varepsilon}} + \frac{\sigma^2 B^2}{\varepsilon^2} \\
K &\geq \smash{\sqrt{\frac{H}{\mu}}}\ln\frac{HB^2}{\varepsilon} + \frac{\sigma^2}{\mu \varepsilon}.
\end{align*}
Therefore, when $\delta \leq \frac{H}{\sqrt{\varepsilon}}$ in the convex setting or $\delta \lesssim \sqrt{H\mu}$ in the strongly convex case, then Algorithm \ref{alg:ISMD} can surpass this lower bound by exploiting the proxy $\F$ in order to achieve a better guarantee than could be achieved by \emph{any} algorithm that only uses $K$ stochastic gradients from $\f$.

Although Algorithm~\ref{alg:ISMD} only requires one stochastic gradient from $\f$ for each iteration, each update does require finding an approximate stationary point of the function $\phi_k$, which is more computationally expensive than one \textsc{SGD} update. That said, in the $\mu$-strongly convex case (the conclusion is similar in the merely convex setting), if we ignore all constant and logarithmic factors, Corollary \ref{cor:optimized-guarantees} shows that for $K := \frac{\delta}{\mu} + \frac{\sigma^2}{\mu\epsilon}$, $G^2 = \mu\epsilon$, and $\eta = \frac{1}{\mu K}$, Algorithm \ref{alg:ISMD} will find an $O(\epsilon)$-approximate minimizer of $\f$. Furthermore, Proposition \ref{prop:making-gradient-small} (with $\rho = \sigma$) shows that the proximal subproblem in each iteration of Algorithm~\ref{alg:ISMD} can be solved to sufficient accuracy using $T := \frac{K}{\mu K} + \frac{H\sigma^2}{\mu^2 K \epsilon}$ steps of \textsc{SGD} on $\phi_k$. Therefore, the total number of gradient update computations needed to implement Algorithm \ref{alg:ISMD} is $TK = \frac{H}{\mu} + \frac{H}{\mu}\cdot\frac{\sigma^2}{\mu \epsilon}$. In contrast, \eqref{eq:comparison-to-sgd-strongly-convex} implies that \textsc{SGD} directly on $\f$ would require $S := \frac{H}{\mu} + \frac{\sigma^2}{\mu \epsilon}$ gradient update computations to reach an $O(\epsilon)$-approximate minimizer of $\f$. 

So, \textsc{SGD} directly on $\f$ could be computationally cheaper than Algorithm \ref{alg:ISMD} in terms of the total number of floating point operations needed since $S < TK$. However, crucially, Algorithm \ref{alg:ISMD} only requires $K$ stochastic gradients from $\f$ itself (the other gradients come from $\F$), while \textsc{SGD} needs $S > K$. So to summarize, executing Algorithm \ref{alg:ISMD} could require more computation than \textsc{SGD} on $\f$, but it can make up for this by using fewer stochastic gradients from $\f$, which could yield significant savings when $\f$ is harder to access. 

\subsection{Connection to Preconditioning}\label{sec:preconditioning}

A useful way to understand Algorithm \ref{alg:ISMD} is to consider the special case where the proxy $\F(\w) = \frac{1}{2}\w\T\PP\w + \mathbf{b}\T\w$ is quadratic. In this case, exactly solving the proximal subproblems in Algorithm \ref{alg:ISMD} amounts to solving
\begin{gather}
\g_k + \nabla \F(\w_{k+1}) - \nabla \F(\w_k) + \frac{1}{\eta}(\w_{k+1} - \w_k) = 0 \nonumber\\
\iff \w_{k+1} = \w_k - \eta\prn*{\mathbf{I} + \eta\PP}^{-1}\g_k.
\end{gather}
That is, each update in Algorithm \ref{alg:ISMD} amounts to a single preconditioned SGD update with preconditioning matrix $(\mathbf{I} + \eta\PP)^{-1}$. When a quadratic proxy that satisfies Assumption \ref{ass:similarity} can be found, then these updates with a fixed preconditioning matrix will be very effective at minimizing the objective. 

However, for many objectives, the Hessian $\nabla^2 \f$ is not constant (nor close to constant) and so it cannot be well-approximated everywhere by a matrix $\PP$ and thus any fixed preconditioner may not be effective. In this more general case, we can view each step of Algorithm \ref{alg:ISMD} as approximating a locally preconditioned SGD update with the preconditioning matrix $(\mathbf{I} + \eta\nabla^2\F(\w_k))^{-1}$. This is because for sufficiently smooth $\F$, $\nabla \F(\w_{k+1}) - \nabla \F(\w_k) \approx \nabla^2\F(\w_k)(\w_{k+1} - \w_k)$ so
\begin{gather}
\g_k + \nabla \F(\w_{k+1}) - \nabla \F(\w_k) + \frac{1}{\eta}(\w_{k+1} - \w_k) = 0 \nonumber\\
\implies \w_{k+1} \approx \w_k - \eta\prn{\mathbf{I} + \eta\nabla^2\F(\w_k)}^{-1}\g_k.
\end{gather}
Of course, this is not a perfect correspondence, but it demonstrates the role of $\F$, which is effectively to adapt each update along the $\g_k$ direction to the local curvature of $\F$. Furthermore, when the curvature of $\F$ is close enough to that of $\f$, this also indicates why Algorithm \ref{alg:ISMD}'s updates should be more effective than SGD directly on $\f$.

\section{Applications}\label{sec:applications}

Algorithm \ref{alg:ISMD} allows us to reduce the number of stochastic gradients needed from the objective $\f$ by approximately solving a sequence of strongly convex subproblems based on $\F$. Accordingly, this approach is most advantageous when $\F$ is a faithful approximation of $\f$ in the sense that $\delta$ is small, and also when the model $\F$ is more accessible than the true objective $\f$. This can arise in numerous situations, and in this section we will discuss several examples. 

\subsection{Regression and classification with costly labels} \label{sec:binary-logistic-regression}

In supervised learning, we want to train using input-label pairs $(x,y) \sim \mc{D}$. Often, we can easily collect a large sample of i.i.d.~inputs $x_1,x_2,\dots$, but obtaining labels for these inputs can be much more costly. For example, labelling whether an MRI scan contains a tumor may require consulting with (and likely paying) a trained medical technician. So, we would like to learn using as few labels as possible, and our algorithm can provide a means of doing this.

For inputs $x \in \R^d$ and labels $y \in \R$, least squares regression, where $\f(\w) = \frac{1}{2}\E_{(x,y)\sim\mc{D}}[(\w\T x - y)^2]$, is one of the main workhorses of statistical analysis. Conveniently, the Hessian of this loss, $\nabla^2 \f(\w) = \E_{x\sim\mc{D}_x}[xx\T]$, does not depend on the labels $y$, so we can construct a proxy $\F(\w) = \frac{1}{2}\w\T \E_{x\sim\mc{D}_x}[xx\T] \w$ that satisfies Assumption \ref{ass:similarity} with $\delta = 0$ without using any labels! Thus, we can implement Algorithm \ref{alg:ISMD} using just one labelled sample $(x_k,y_k)$ per iteration to estimate $\g_k = (\w\T x_k - y_k)x_k$, and then approximately minimizing $\phi_k$ only requires $x$ samples. 


Similarly, for inputs $x\in\R^d$ and binary labels $y \in \crl{0,1}$, the logistic regression loss function is
\begin{equation*}
\f(\w) = -\underset{(x,y)\sim\mc{D}}{\E}\brk*{y\ln s(\w\T x) + (1-y)\ln (1-s(\w\T x))},
\end{equation*}
where $s(z) = 1/(1+\exp(-z))$ denotes the logistic function. As in the previous example, the Hessian of this loss is
\begin{equation*}
\nabla^2 \f(\w) = \underset{(x,y)\sim\mc{D}}{\E}\brk*{s(\w\T x)(1-s(\w\T x)) xx\T} ,
\end{equation*}
which again does not depend on the labels $y$. Therefore, we can construct a proxy loss by simply assigning all points the label $y=1$ (or assigning them any arbitrary labels):
\begin{equation*}
\F(\w) := -\underset{x\sim\mc{D}_x}{\E}[\ln s(\w\T x_n)].
\end{equation*}
This ensures $\nabla^2\F = \nabla^2\f$ so that Assumption \ref{ass:similarity} holds with $\delta = 0$, despite using no labels. So again, Algorithm  \ref{alg:ISMD} can be implemented using as little as one label per iteration, plus sufficient $x$ samples to approximately minimize $\phi_k$.


\subsection{Synthetic data, simulators, and meta-learning}

In many machine learning problems of interest, it is difficult to collect i.i.d.~samples from the target distribution. For instance, as discussed in the previous section, obtaining a label in medical imaging applications may require a consultation with a radiologist and perhaps even an additional biopsy in difficult cases. This can make it costly and time-consuming to obtain a large training set. 
To mitigate this issue, there has been great interest in using synthetic data as a complement to or substitute for costly samples from the target distribution. Recent advances in generative models like generative adversarial networks (GANs), variational autoencoders, diffusion models, etc.~have raised the possibility of using a small seed of real data to generate an essentially unlimited quantity of synthetic training data. For instance, \citet{frid2018synthetic} used a GAN to augment a small dataset of 182 CT scans of liver lesions, which allowed them to train a classifier with better performance. Similarly, meta-learning \citep{metalearningsurvey}, multitask learning \citep{zhang2021survey}, and transfer learning \citep{weiss2016survey} are all efforts to improve performance on a given task by leveraging data collected for other, related tasks. For example, when learning to classify images of giraffes, it is common to use another image dataset like ImageNet for pretraining or joint training---even though ImageNet does not contain labelled images of giraffes. 

In addition to possible difficulties in obtaining the training data, evaluating the loss or its gradient can be expensive. In robotics, evaluating the reward achieved by a given policy might require waiting for a physical robot to execute the policy for a period of time, and waiting seconds or minutes for each function or gradient evaluation can greatly slow training. In addition, training something like a self-driving car in the real world carries the risk of damaging the vehicle or injuring pedestrians. As a result, it is very common to use simulators like MuJoCo \citep{todorov2012mujoco} and the OpenAI Gym \citep{brockman2016openai}, which make it possible to train models in a virtual environment, allowing for faster, risk-free processing. 

Synthetic data and simulators are often used to define a surrogate loss function that is directly minimized to train the model. In meta-learning, multitask learning, or transfer learning, other tasks are usually used as a prior that biases training towards models that also perform well on all tasks simultaneously. Therefore, all of these approaches will only work to the extent that the minima of the surrogate losses correspond to minima of the target loss $\f$ itself. Indeed, recent work has cast some doubt on the utility of, e.g., ImageNet pretraining \citep{he2019rethinking,kornblith2019better} and robots trained on simulators frequently generalize poorly to the real world \citep{zhao2020sim}, presumably because these surrogates can bias the model towards worse solutions.
In contrast, our approach does not rely on minima of the proxy having any particular correspondence with those of the target; instead, we use the curvature of the proxy to speed up optimization of the objective itself.

\subsection{Learning with public and private data}

Another natural setting where our algorithm could be useful is when the objective $\f$ is based on private data. One approach to learning from this data while maintaining rigorous privacy guarantees is to use a differentially private algorithm to generate a synthetic dataset, which could then be used freely as described in the previous section \citep[see, e.g.,][]{amin2019differentially,bowen2021comparative} 

Our method might also be useful when a stock of non-private data is be available in addition to the sensitive data. For example, some patients could elect to provide their medical data to researchers without restriction, while others only do so on the condition that their privacy is rigorously protected. Similarly, different individuals may be subject to different regulatory regimes---data from E.U.~residents must be used in accordance with more stringent GDPR rules that do not apply to, e.g., US residents. In these cases, one could use just the non-private data to learn, and avoid the trouble of using the protected data, but this could be suboptimal if a large fraction of the total data is private. Instead, we could use the non-private data to define $\F$ while using a privacy-preserving mechanism to compute stochastic gradients of $\f$ from the protected samples \citep{dwork2014algorithmic}. 
When the private and non-private samples have similar distributions, we would expect Assumption \ref{ass:similarity} to hold with small $\delta$, meaning that relatively few iterations of Algorithm \ref{alg:ISMD} will be needed to achieve low error, and therefore the private data would need to be accessed fewer times. Since privacy guarantees degrade with the number of accesses, this could allow for better performance and better privacy guarantees than would be possible using the private data alone. 

\subsection{Handling distribution shifts}

Often, we would like to learn in the face of distribution shift---a situation where the training distribution differs from the test distribution. For example, the distribution of inputs and outputs could be changing over time, and the training set is collected today and then the model is deployed tomorrow. In recent work, \citet{lazaridou2021mind} argue that the performance of neural language models can degrade over time for this very reason. To keep a model up-to-date, one could periodically collect a new dataset and train a whole new model. However, collecting a large enough quantity of high-quality data could be expensive and time consuming, and it would be better to take advantage of the older data that is already available. To apply our algorithm, we could have $\f$ be the loss on the current data distribution, while out-of-date data could be used to define $\F$. As long as the distribution shift is not too dramatic, we could expect Assumption \ref{ass:similarity} to hold with a relatively small $\delta$, and so to implement our method, only a small amount of new data would needed.

Distribution shifts can also arise when the training distribution is stationary but biased. For example, survey data is often affected by selection, acquiescence, or social desirability biases which can lead to differences between survey results and reality. Relatedly, when learning from a data source in which one group or class is rare, we may up-weight the loss on the rare group in order to prevent the model from learning to ignore it, which corresponds to targeting a test distribution where that group's prevalence is higher. In either case, our algorithm could be applied as described above using relatively few unbiased samples.

\section{Experiments}\label{sec:experiments}

To exhibit the effectiveness of our method in practice, we conducted several experiments on realistic problems. 
\begin{figure*}[t]
\centering
\includegraphics[width=0.3\textwidth]{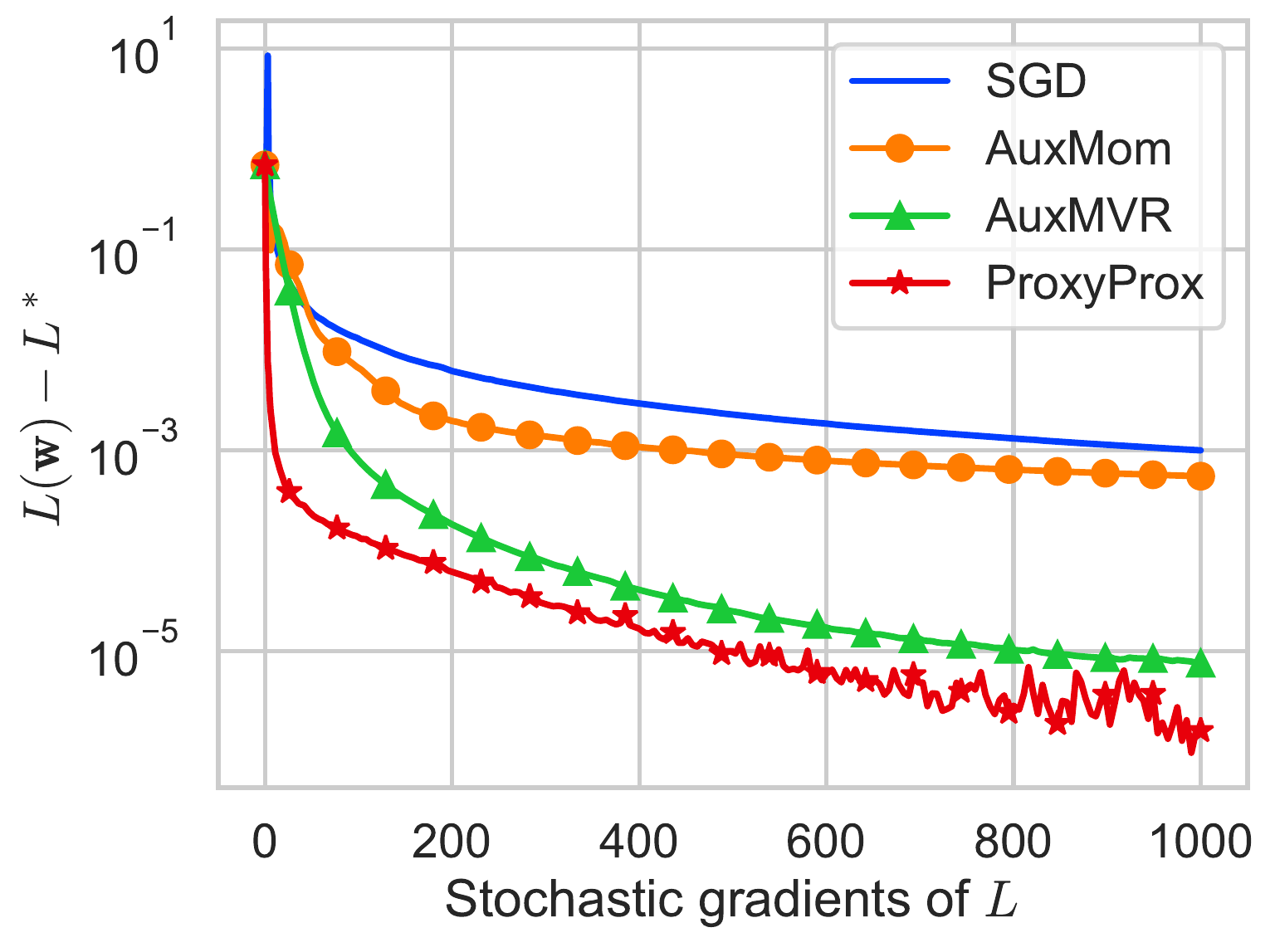}
\includegraphics[width=0.3\textwidth]{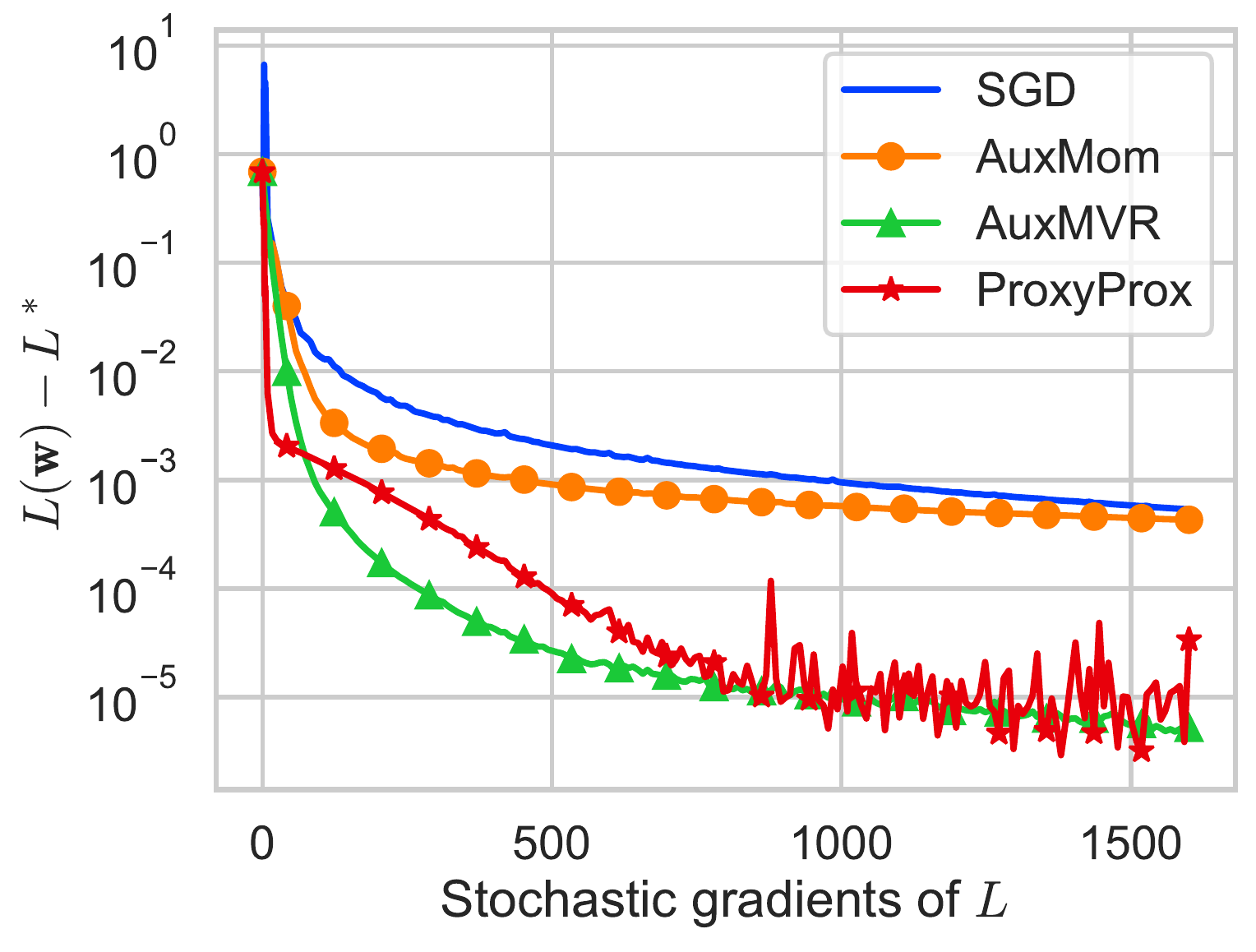}
\includegraphics[width=0.3\textwidth]{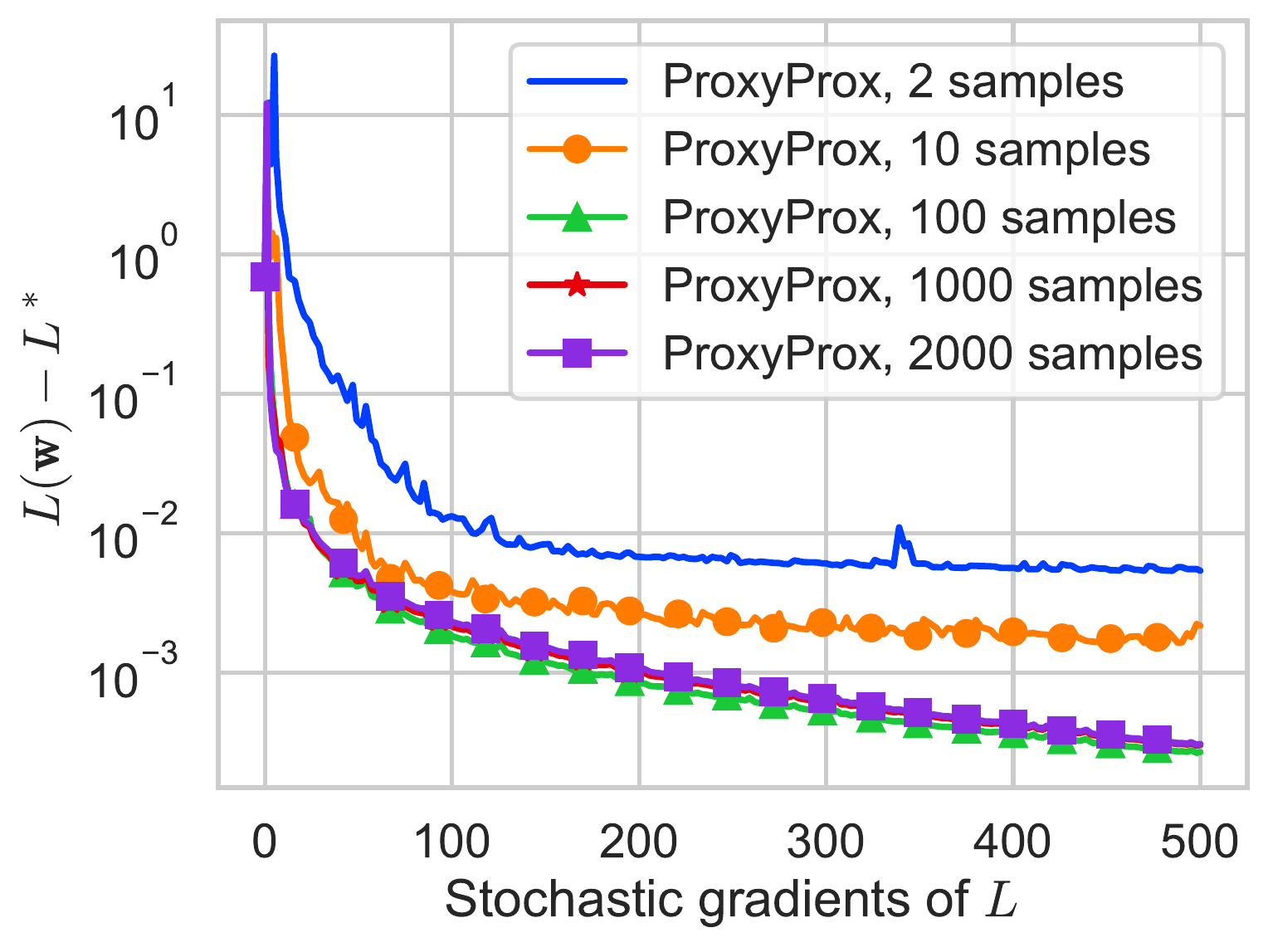}
\vspace{-4mm}
\caption{Binary logistic regression on `mushrooms'. Left: batch size 1024; middle: batch size 256; right: performance of \textsc{ProxyProx} with $\F$ subsampled from $\f$, stochastic gradients of $\f$ computed with batch size 256.}
\vspace{-1mm}
\label{fig:logistic_fake}
\end{figure*}

\textbf{Logistic regression:} 
First, we consider a simple binary logistic regression problem with the `mushrooms' dataset, which consists of $8124$ samples of dimension $112$. In this experiment, the objective $\f$ is the logistic regression loss evaluated on the training data plus an $\ell_2$ regularizer with weight $\mu=10^{-6}H$, where $H$ is the smoothness parameter of the objective. 
Stochastic gradients for $\f$ are calculated by evaluating the gradient on a minibatches of size 256 or 1024 drawn uniformly with replacement. To simulate a situation where labels are hard to come by as discussed in Section~\ref{sec:binary-logistic-regression}, we define the proxy $\F$ by assigning random labels to the input features. The results, depicted in Figure \ref{fig:logistic_fake}, show that our method  \textsc{ProxyProx} has competitive performance compared with performing \textsc{SGD} directly on the objective. We are also competitive with the \textsc{AuxMom} and \textsc{AuxMVR} algorithms proposed by \citet{chayti2022optimization}, which are momentum and variance-reduced analogues of \textsc{ProxyProx}. 
In Figure \ref{fig:logistic_fake}, the $x$-axis counts how many stochastic gradients from $\f$ have been used by each method, so this comparison holds constant the amount of access to $\f$, although the computational cost of \textsc{ProxyProx}, \textsc{AuxMom} and \textsc{AuxMVR} is higher due to the approximate proximal-point updates in each iteration. We approximately solve each of the $\phi_k$ subproblems using gradient descent. For each algorithm in the experiment, we selected hyperparameters like $\eta$ using grid search to minimize the loss after 250 steps and ran the best option for the full 1000 iterations. We set \textsc{AuxMom} and \textsc{AuxMVR}'s $a=0.1$, which corresponds to the standard momentum setting of $0.9$.

\begin{figure*}[t]
\centering
\includegraphics[width=0.28\textwidth]{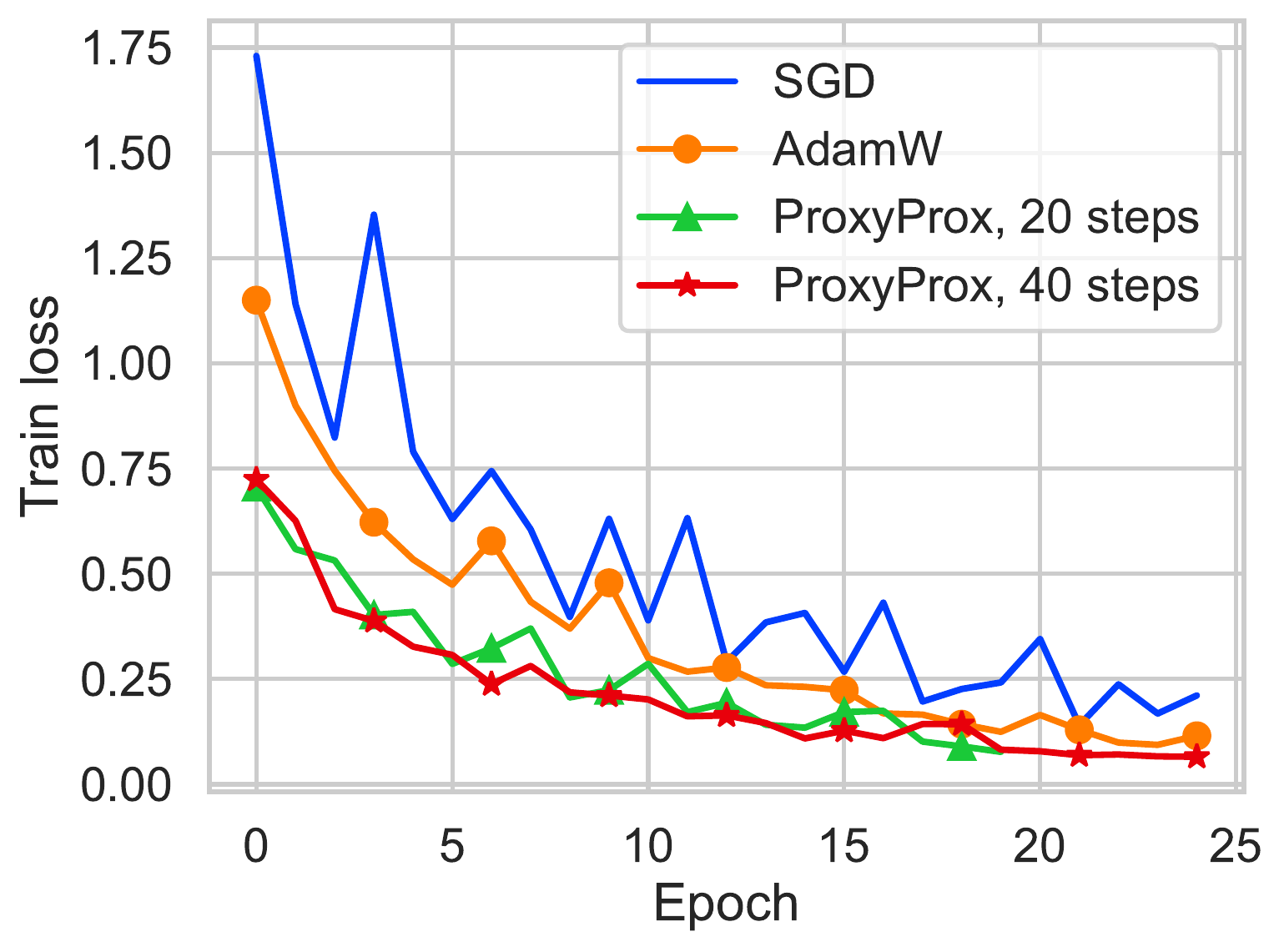}
\includegraphics[width=0.28\textwidth]{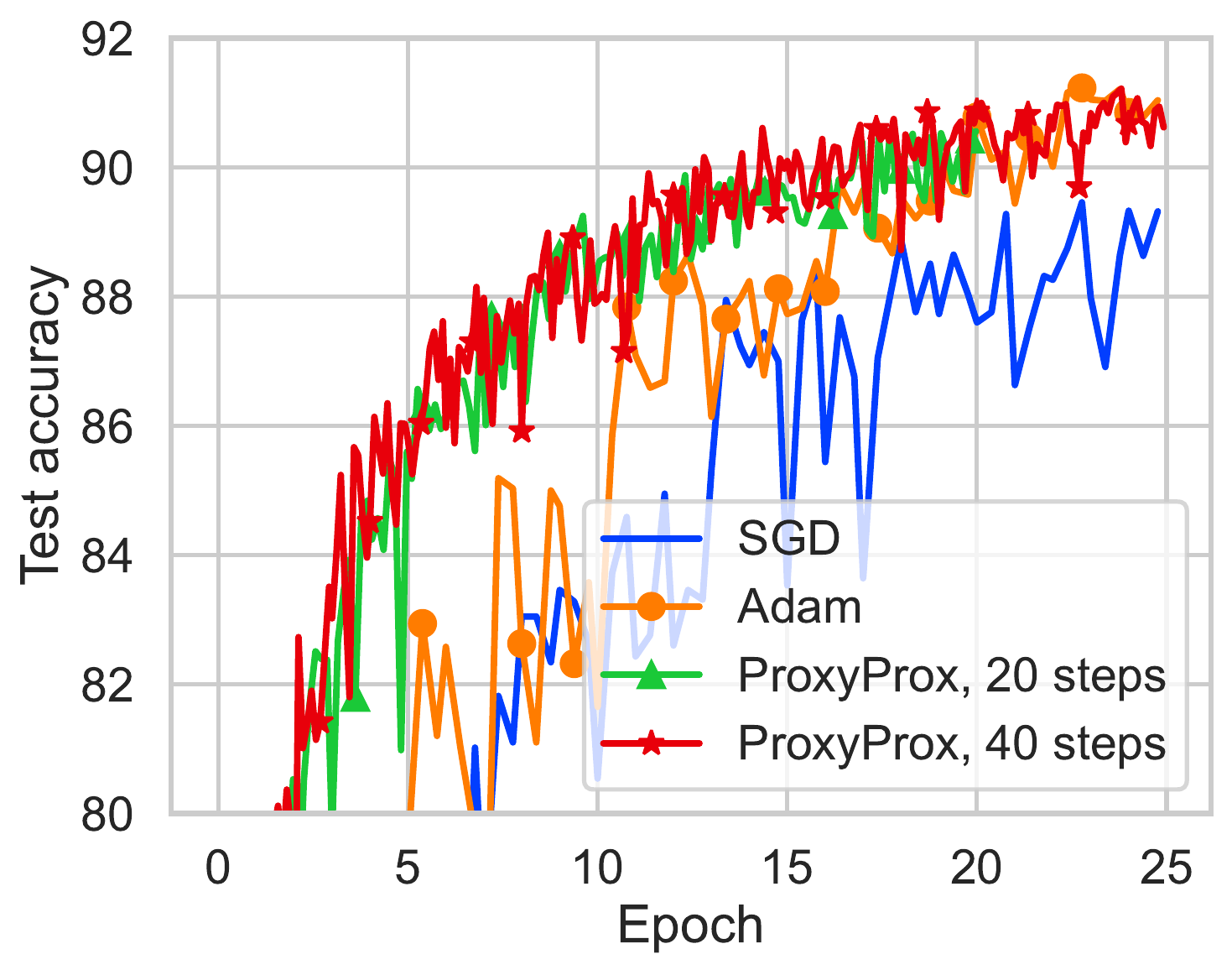}
\includegraphics[width=0.28\textwidth]{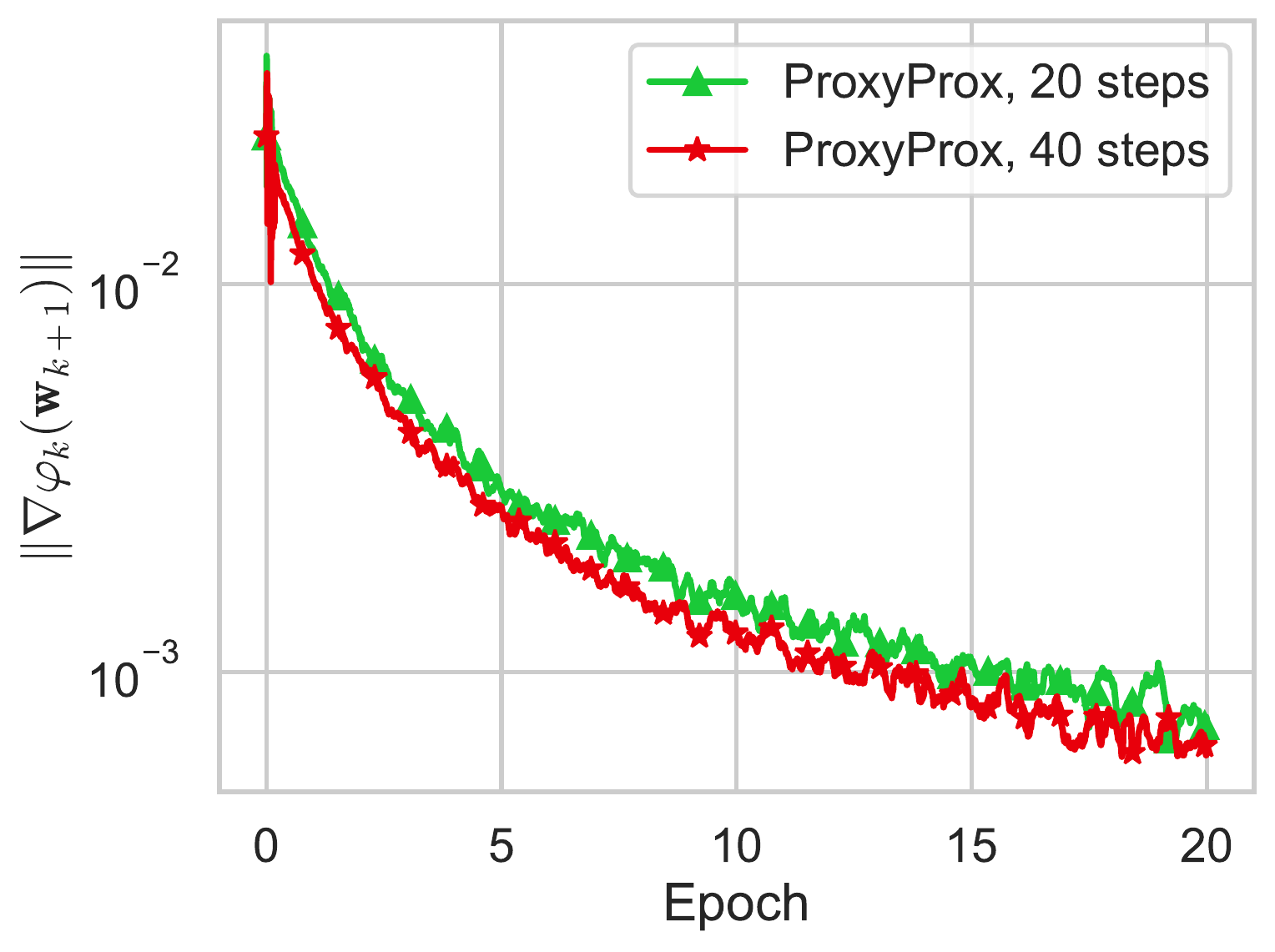}
\vspace{-3mm}
\caption{CIFAR-10 experiment. Left: the train loss, $\f$; middle: the test accuracy; right: $\|\nabla \phi_k(\w_{k+1})\|$ after optimizing $\phi_k$ with $20$ or $40$ \textsc{SGD} steps. The x-axis represents the number of passes over the full dataset defining $\f$. 
}
\vspace{-2mm}
\label{fig:NN}
\end{figure*}

\textbf{ResNet-18:} In this experiment, summarized in Figure \ref{fig:NN}, we train a ResNet-18 network on CIFAR-10, defining $\F$ using a $2560$ subset of the images, and using the full train dataset of $50000$ images for $L$. While $\f$ corresponds to the training loss, we also plot the test loss to show that our method does not overfit. We compare running our method with either $20$ and $40$ iterations of \textsc{SGD} (equivalent to 1 epoch and 2 epochs on the $\F$ data) to approximately minimize $\phi_k$, with the stepsize tuned by grid search and ultimately set to $0.01$. Figure \ref{fig:NN} indicates that a single pass over the subsampled data seems sufficient to minimize $\phi_k$. We use minibatch size of 128 for both $\g_k$ and the SGD updates to minimize $\phi_k$. Our method can be better than \textsc{AdamW} for the first $\sim 25$ epochs, and it improves more convincingly over SGD. We use standard stepsizes: $0.1$ for \textsc{SGD}, $0.001$ for \textsc{AdamW}, and weight decay of $0.1$ for \textsc{AdamW}, which gave the best test accuracy in a grid search, and all methods also used cosine annealing. We note that just training on $\F$ using \textsc{SGD} plateaus at test accuracy 65 after just a few passes, so, our method's improvement is indeed coming from using $\F$ jointly with $\f$.

\section{Open problems and future directions}\label{sec:conclusion}

In the convex setting, we show in Section \ref{sec:convergence-analysis} that under Assumptions \ref{ass:similarity} and \ref{ass:noise}, Algorithm \ref{alg:ISMD} can take advantage of the proxy to achieve a convergence rate comparable to SGD on a $\delta$-smooth function---despite the fact that $\f$ may not be smooth, and even if it is smooth, its smoothness parameter may be much larger than $\delta$. However, the complexity guarantees in Corollary \ref{cor:optimized-guarantees} require knowledge of problem parameters, including $\delta$ which may be hard to estimate in practice. In future work, it may therefore be useful to explore whether similar performance could be achieved using a method that is adaptive to unknown problem parameters. Also, the uniform bound on the stochastic gradient variance in Assumption \ref{ass:noise} can be unrealistic, and in other contexts it can be possible to show similar results under the weaker assumption that the gradient variance is bounded only at $\w^*$ \citep{moulines2011non,dragomir2021fast}.

Going beyond the convex setting is also important in machine learning, where many models used in practice, like neural networks, give rise to non-convex training objectives. Conceptually, the idea of using the proxy $\F$ to guide optimization of the objective $\f$ seems sound even when these functions are non-convex, and our CIFAR-10 experiment shows that it can work well. However, there is still work to do in proving theoretical guarantees. We can show that Algorithm \ref{alg:ISMD} converges to the vicinity of a stationary point: 
\begin{restatable}{theorem}{nonconvex}\label{thm:non-convex}
Under Assumptions \ref{ass:similarity} and \ref{ass:noise}, with $\f$ differentiable but potentially non-convex, let $\eta \leq \frac{1}{4\delta}$. If
\[
\E\nrm{\phi_k(\w_{k+1})}^2 
\leq \frac{7}{16\eta^2}\E\nrm{\w_{k+1} - \w_k}^2 + \frac{1}{8}\nrm{\nabla \f(\w_{k+1})}^2,
\]
and $\E\phi_k(\w_{k+1}) \leq  \E\phi_k(\w_k)$ for each $k$, then
\[
\frac{1}{K}\,\smash{\sum_{k=1}^K}\, \E\nrm{\nabla \f(\w_k)}^2 
\leq \frac{48(\f(\w_0) - \f^*)}{\eta K} + 8\sigma^2.
\]
\end{restatable}
We prove this in Appendix \ref{app:non-convex}.
Taking $\eta = \frac{1}{4\delta}$, this first term becomes $O(\delta(\f(\w_0) - \f^*) / K)$, which matches the rate of gradient descent on a $\delta$-smooth objective, paralleling our results in the convex setting. However, the noise term remains $\Omega(\sigma^2)$ regardless of $\eta$ or $K$. It is not clear whether this is merely a shortcoming in our analysis, or if this actually reflects the worst case performance of Algorithm \ref{alg:ISMD}. To our knowledge, a  satisfying analysis of stochastic proximal-point methods in the non-convex setting is lacking more generally, so this situation may not be unique to our method.

\subsection*{Acknowledgements}

We acknowledge support from the French government under the management of the Agence Nationale de la Recherche as part of the ``Investissements d'avenir'' program, reference ANR-19-P3IA-0001 (PRAIRIE 3IA Institute), as well as from the European Research Council (grant SEQUOIA 724063).

\bibliographystyle{icml2023}
\bibliography{biblio}
\newpage
\appendix
\onecolumn
\section{Proof of Theorem \ref{thm:strongly-convex}}\label{app:proof-of-thmstronglyconvex}

Our analysis requires that $h = \f - \F$ be differentiable under Assumption \ref{ass:similarity}, but when $\f$ is convex, it is not necessary for $\f$ itself to be differentiable. In particular, Assumption \ref{ass:noise} requires that the expectation of each stochastic gradient $\g_k$ must be a subgradient of $\f$ at $\w_k$. Accordingly, in the following proofs, we will use $\nabla \f(\w_k) := \E[\g_k\,|\,\w_k]$ as notation to denote the particular subgradient corresponding to the expectation of $\g_k$. Since $h = \f - \F$ is differentiable, we can use this, in turn, to define $\nabla \F(\w_k) := \nabla \f(\w_k) - \nabla \h(\w_k)$. Finally, since $0 \in \partial \f(\w^*)$, we denote $\nabla \f(\w^*) := 0$ even when $\f$ is not differentiable at its minimizer.

\begin{restatable}{lemma}{mainrecurrence}\label{lem:one-step-general}
Let $\w^* \in \argmin_{\w} \f(\w)$ and $\eta \leq \frac{1}{4\delta}$. Then under Assumptions \ref{ass:similarity} and \ref{ass:noise}, for any $\w_{k+1},\w_k$
\begin{align*}
\E[\f(\w_{k+1}) - \f^*] 
&\leq \frac{1}{2\eta}\E\nrm*{\w_k - \w^*}^2 - \E D_{\h}(\w^*;\w_k) \\
&\quad- \frac{1}{2\eta}\E\nrm*{\w_{k+1} - \w^*}^2 + \E D_{\h}(\w^*;\w_{k+1}) \\
&\quad+ \E\inner{\nabla \phi_k(\w_{k+1})}{\w_{k+1} - \w^*} + 2\eta\sigma^2 \\
&\quad- \frac{1}{4\eta}\E\nrm*{\w_{k+1} - \w_k}^2 - \E D_{\f}(\w^*;\w_{k+1}),
\end{align*}
where the expectation is taken over the randomness in $\g_k$.
\end{restatable}
\begin{proof}
Computing $\nabla \phi_k$ and rearranging yields
\begin{align}
\nabla \phi_k(\w_{k+1})
&= \nabla \F(\w_{k+1}) + \g_k - \nabla \F(\w_k) + \frac{1}{\eta}(\w_{k+1} - \w_k) \notag \\ 
\implies \nabla \f(\w_{k+1}) &= \nabla \h(\w_{k+1}) - \nabla \h(\w_k)  + \nabla \phi_k(\w_{k+1}) + \nabla \f(\w_k) - \g_k - \frac{1}{\eta}(\w_{k+1} - \w_k).\label{eq:expression-for-nablaLk}
\end{align}
Also, \eqref{eq:def-bregman-divergence} immediately implies
\begin{equation}\label{eq:bregman-rearrangement-thing}
\f(\w_{k+1}) - \f^* = \inner{\nabla \f(\w_{k+1})}{\w_{k+1} - \w^*} - D_{\f}(\w^*;\w_{k+1}).
\end{equation}
After substituting \eqref{eq:expression-for-nablaLk} into \eqref{eq:bregman-rearrangement-thing}, we derive
\begin{align}
&\f(\w_{k+1}) - \f^* \nonumber\\
&= \inner{\nabla \f(\w_{k+1})}{\w_{k+1} - \w^*} - D_{\f}(\w^*;\w_{k+1}) \notag \\
&= \inner{\nabla \h(\w_{k+1}) - \nabla \h(\w_k)  + \nabla \phi_k(\w_{k+1}) + \nabla \f(\w_k) - \g_k + \frac{1}{\eta}(\w_k - \w_{k+1})}{\w_{k+1} - \w^*} - D_{\f}(\w^*;\w_{k+1}) \notag \\
&= D_{\h}(\w^*;\w_{k+1}) - D_{\h}(\w^*;\w_k) + D_{\h}(\w_{k+1};\w_k) + \inner{\nabla \phi_k(\w_{k+1}) + \nabla \f(\w_k) - \g_k}{\w_{k+1} - \w^*}\nonumber\\
&\qquad\qquad+ \frac{1}{2\eta}\nrm*{\w_k - \w^*}^2 - \frac{1}{2\eta}\nrm*{\w_{k+1} - \w^*}^2 - \frac{1}{2\eta}\nrm*{\w_{k+1} - \w_k}^2 - D_{\f}(\w^*;\w_{k+1}), \label{eq:lemma-proof-penultimate}
\end{align}
where the third equality uses the three-point property with $D_h$, \eqref{eq:def-three-point-identity}. To proceed, Assumption \ref{ass:noise}, Young's inequality, and $\eta \leq \frac{1}{4\delta}$, together imply
\begin{align*}
\E\inner{\nabla \f(\w_k) - \g_k}{\w_{k+1} - \w^*}
&= \E\inner{\nabla \f(\w_k) - \g_k}{\w_{k+1} - \w_k} \\
&\leq \frac{\eta\E\nrm{\nabla \f(\w_k) - \g_k}^2}{1-2\eta\delta} + \frac{1-2\eta\delta}{4\eta}\nrm{\w_{k+1} - \w_k}^2 \\
&\leq 2\eta\sigma^2 + \frac{1-2\eta\delta}{4\eta}\nrm{\w_{k+1} - \w_k}^2,
\end{align*}
where the expectation is taken over the randomness in $\g_k$. Since $h$ is $\delta$-smooth, in light of \eqref{eq:smoothness-inequality}, $\abs{D_h(\w;\w_k)} \leq \frac{\delta}{2}\nrm{\w - \w_k}^2$, so 
\begin{equation*}
\E\brk*{D_h(\w_{k+1};\w_k) + \inner{\nabla \f(\w_k) - \g_k}{\w_{k+1} - \w^*}}
\leq 2\eta\sigma^2 + \frac{1}{4\eta}\E\nrm{\w_{k+1}-\w_k}^2.
\end{equation*}
Taking the expectation of \eqref{eq:lemma-proof-penultimate} and plugging this in
completes the proof.
\end{proof}

\stronglyconvexthm*
\begin{proof}
For brevity, denote $\Delta_k := \frac{1}{2\eta}\nrm{\w_k - \w^*}^2 - D_{\h}(\w^*;\w_k)$. Because $\f$ is strongly convex, it follows immediately from \eqref{eq:def-strong-convexity} and \eqref{eq:def-bregman-divergence} that $D_{\f}(\w^*;\w_{k+1}) \geq \frac{\mu}{2}\nrm{\w_{k+1} - \w^*}^2$.
Also, in light of \eqref{eq:smoothness-inequality}, it follows under Assumption \ref{ass:similarity} that 
\begin{equation*}
\abs{D_{\h}(\w^*;\w_{k+1})} \leq \frac{\delta}{2}\nrm{\w_{k+1} - \w^*}^2 \leq \frac{1}{8\eta}\nrm{\w_{k+1} - \w^*}^2.
\end{equation*}
This implies that 
\begin{equation*}
0 \leq \Delta_{k+1} \leq \frac{5}{8\eta}\nrm{\w_{k+1} - \w^*}^2 \leq \frac{5}{4\eta\mu}D_L(\w^*;\w_{k+1}).
\end{equation*}

Therefore, by Lemma \ref{lem:one-step-general} and Young's inequality
\begin{align}
&\E[\f(\w_{k+1}) - \f^*] \nonumber\\
&\leq \E\Delta_k - \E\Delta_{k+1} + 2\eta\sigma^2 + \E\inner{\nabla \phi_k(\w_{k+1})}{\w_{k+1} - \w^*}  - \frac{1}{4\eta}\E\nrm*{\w_{k+1} - \w_k}^2 - \E D_{\f}(\w^*;\w_{k+1}) \notag \\
&\leq \E\Delta_k - \prn*{1 + \frac{2\eta\mu}{5}}\E\Delta_{k+1} + 2\eta\sigma^2 + \E\inner{\nabla \phi_k(\w_{k+1})}{\w_{k+1} - \w^*} - \frac{1}{4\eta}\E\nrm*{\w_{k+1} - \w_k}^2 - \frac{\mu}{4}\E\nrm{\w_{k+1} - \w^*}^2. \label{eq:nabla-phi-cancellation}
\end{align}
So by Young's inequality and the assumed upper bound
\begin{equation}
\nrm{\nabla\phi_k(\w_{k+1})}^2 \leq \frac{\mu}{4\eta}\E\nrm*{\w_{k+1} - \w_k}^2 + G^2
\end{equation}
we conclude
\begin{align}
\E\inner{\nabla \phi_k(\w_{k+1})}{\w_{k+1} - \w^*}
&\leq \frac{1}{\mu}\E\nrm*{\nabla \phi_k(\w_{k+1})} + \frac{\mu}{4}\E\nrm{\w_{k+1} - \w^*}^2 \\
&\leq \frac{1}{4\eta}\E\nrm*{\w_{k+1} - \w_k}^2 + \frac{\mu}{4}\E\nrm{\w_{k+1} - \w^*}^2 + \frac{G^2}{\mu}.
\end{align}
Plugging this into \eqref{eq:nabla-phi-cancellation} yields
\begin{equation}\label{eq:strongly-convex-descent}
\E[\f(\w_{k+1}) - \f^*] 
\leq \E\Delta_k - \prn*{1 + \frac{2\eta\mu}{5}}\E\Delta_{k+1} + 2\eta\sigma^2 + \frac{G^2}{\mu} .
\end{equation}

As is typical in strongly convex stochastic optimization, we now introduce weights $\alpha_k = \prn{1 + \frac{2\eta\mu}{5}}^{k-1}$ \citep[see Section 5 in][for a similar weighting]{nesterov2008confidence} with sum $A_K := \sum_{k=1}^K \alpha_k$ and we will attempt to upper bound the suboptimality of the weighted iterate 
\begin{equation}\label{eq:weighted-iterate}
\bar{\w}_K := \frac{1}{A_K}\sum_{k=1}^K\alpha_k \w_k.
\end{equation}
By the convexity of $\f$ and \eqref{eq:strongly-convex-descent},
\begin{align*}
\E \f(\bar{\w}_K) - \f^* 
&\leq \frac{1}{A_K}\sum_{k=1}^K\alpha_k\E[\f(\w_k) - \f^*] \\
&\leq \frac{\E\Delta_0}{A_K} + 2\eta\sigma^2 + \frac{G^2}{\mu} \\
&\leq \frac{5\E\nrm{\w_0 - \w^*}^2}{8\eta A_K} + 2\eta\sigma^2 + \frac{G^2}{\mu}.
\end{align*}
Finally, we lower bound
\begin{equation*}
A_K = \frac{5\prn*{\prn*{1 + \frac{2\eta\mu}{5}}\smash{^{K}} - 1}}{2\eta\mu} \geq \prn*{1 + \frac{2\eta\mu}{5}}^{K-1} ,
\end{equation*}
which completes the proof.
\end{proof}

\section{Proof of Theorem \ref{thm:convex}}\label{app:proof-of-thmconvex}
\thmconvex*
\begin{proof}
By construction, $\f^{(\mu)}$ is $\mu$-strongly convex, and yet $\f^{(\mu)} - \F^{(\mu)} = h$, so this difference remains $\delta$-smooth. Therefore, our analysis from Theorem \ref{thm:strongly-convex} can be applied to $\f^{(\mu)}$. Denote $\w^*_{\mu} := \argmin_\w \f^{(\mu)}(\w)$. In the course of proving Theorem \ref{thm:strongly-convex}, we showed in \eqref{eq:strongly-convex-descent} that
\begin{align*}
\E[\f^{(\mu)}(\w_{k+1}) - \f^{(\mu)*}] 
&\leq \E\Delta_k - \prn*{1 + \frac{2\eta\mu}{5}}\E\Delta_{k+1} + 2\eta\sigma^2 + \frac{G^2}{\mu}  \\
&\leq \E\Delta_k - \E\Delta_{k+1} + 2\eta\sigma^2 + \frac{G^2}{\mu},
\end{align*}
where $\Delta_k := \frac{1}{2\eta}\nrm{\w_k - \w^*_\mu}^2 - D_h(\w^*_\mu;\w_k)$ and $0 \leq \Delta_k \leq \frac{5}{8\eta}\nrm{\w_k - \w^*_\mu}^2$ for all $k$. Furthermore, $\f^{(\mu)}(\w_{k+1}) \geq \f(\w_{k+1})$ and
\begin{equation*}
\f^{(\mu)*} \leq \f^{(\mu)}(\w^*) = \f^* + \frac{\mu}{2}\nrm{\w_0 - \w^*}^2.
\end{equation*}
Therefore, by the convexity of $\f$
\begin{align}
\E\f\prn*{\frac{1}{K}\sum_{k=1}^K \w_k} - \f^* 
&\leq \frac{1}{K}\sum_{k=1}^K\E\brk*{\f(\w_k) - \f^*} \notag \\
&\leq \frac{1}{K}\sum_{k=1}^K\E\brk*{\f^{(\mu)}(\w_k) - \f^{(\mu)*} + \frac{\mu}{2}\nrm{\w_0 - \w^*}^2} \notag \\
&\leq \frac{\E\Delta_0}{K} + 2\eta\sigma^2 + \frac{G^2}{\mu} + \frac{\mu}{2}\E\nrm{\w_0 - \w^*}^2. \label{eq:convex-penultimate-step}
\end{align}
Finally,
\begin{equation*}
0 = \nabla \f^{(\mu)}(\w^*_\mu) = \nabla \f(\w^*_\mu) + \mu(\w^*_\mu - \w_0),
\end{equation*}
so,
\begin{align*}
\nrm*{\w_0 - \w^*_\mu}^2 - \nrm*{\w_0 - \w^*}^2
&= -\nrm*{\w^* - \w^*_\mu}^2 + 2\inner{\w_0 - \w^*_\mu}{\w^* - \w^*_\mu} \\
&= -\nrm*{\w^* - \w^*_\mu}^2 - \frac{2}{\mu}\inner{\nabla \f(\w^*_\mu)}{\w^*_\mu - \w^*} \\
&\leq 0 .
\end{align*}
where for the last line we used the convexity of $\f$. So,
\begin{equation*}
\Delta_0 \leq \frac{5}{8\eta}\nrm{\w_0 - \w^*_\mu}^2
\leq \frac{5}{8\eta}\nrm{\w_0 - \w^*}^2.
\end{equation*}
When combined with \eqref{eq:convex-penultimate-step}, this means
\begin{align*}
\E\f\prn*{\frac{1}{K}\sum_{k=1}^K \w_k} - \f^* 
&\leq \frac{5\E \nrm{\w_0 - \w^*}^2}{8\eta K} + 2\eta\sigma^2 + \frac{G^2}{\mu} + \frac{\mu}{2}\E\nrm{\w_0 - \w^*}^2,
\end{align*}
and plugging in our choice of $\mu = \frac{1}{\eta K}$ completes the proof.
\end{proof}

\section{Proof of Corollary \ref{cor:optimized-guarantees}}\label{app:corollary}

\stepsizecorollary*
\begin{proof}
The complexity guarantee in the convex case follows immediately after plugging the chosen $\eta$ into the guarantee of Theorem \ref{thm:convex} and then choosing $K$ large enough that the expected suboptimality is less than $\varepsilon$.

Moving on to the strongly convex setting, to apply Theorem \ref{thm:strongly-convex}, we require $\eta \leq \frac{1}{4\delta}$. Suppose for now that $\eta$ also satisfies
\begin{equation*}
\eta \geq \frac{1}{2\mu (K-1)}.
\end{equation*}
Then the first term in the guarantee from Theorem \ref{thm:strongly-convex} is at most
\begin{equation*}
\frac{5B^2}{8\eta}\prn*{1+\frac{2\eta\mu}{5}}^{1-K} 
\leq \frac{5B^2\mu(K-1)}{4}\prn*{1+\frac{2\eta\mu}{5}}^{1-K}.
\end{equation*}
Next, we note that for $g(z) = \ln(1 + z)$, and any $z \geq 0$
\begin{equation*}
g'(z) = \frac{1}{1+z} = 1 - \frac{z}{1+z} \geq 1 - z.
\end{equation*}
Therefore, for $z \geq 0$
\begin{equation*}
\ln(1+z) = \int_0^z g'(t) dt \geq \int_0^z (1 - t) dt = z - \frac{1}{2}z^2.
\end{equation*}
Thus,
\begin{equation*}
\prn*{1+\frac{2\eta\mu}{5}}^{1-K}
= \exp\prn*{-(K-1)\ln\prn*{1+\frac{2\eta\mu}{5}}}
\leq \exp\prn*{-\frac{2(K-1)\eta\mu}{5}\prn*{1 - \frac{\eta\mu}{5}}}
\end{equation*}
therefore, as long as we choose $\eta \leq \frac{5}{2\mu}$, then it follows that
\begin{equation*}
\frac{5B^2}{8\eta}\prn*{1+\frac{2\eta\mu}{5}}^{1-K}
\leq \frac{5B^2\mu(K-1)}{4}\exp\prn*{-\frac{(K-1)\eta\mu}{5}}.
\end{equation*}
Therefore, if all of the following constraints on $\eta$ are satisfied:
\begin{gather}
\eta \leq \frac{1}{4\delta} \label{eq:cor-delta-term}\\
\eta \leq \frac{5}{2\mu} \label{eq:cor-lambda-term}\\
\eta \geq \frac{1}{2\mu(K-1)} \label{eq:cor-no-log-term}\\
\eta \geq \frac{5}{\mu (K-1)}\ln\prn*{\frac{5B^2\mu (K-1)}{4\varepsilon}} \label{eq:cor-log-term}
\end{gather}
then this first term in the guarantee from Theorem \ref{thm:strongly-convex} is upper bounded by $\varepsilon$. We therefore set
\begin{equation}\label{eq:cor-sc-stepsize}
\eta = \frac{5}{\mu (K-1)}\prn*{1 + \ln\prn*{\frac{5B^2\mu (K-1)}{4\varepsilon}}}.
\end{equation}
This value always satisfies \eqref{eq:cor-no-log-term} and \eqref{eq:cor-log-term}, and we will proceed to choose $K$ large enough that it also satisfies \eqref{eq:cor-delta-term} and \eqref{eq:cor-lambda-term}. In particular, we require $K$ such that 
\begin{equation} \label{eq:cor-first-term-K}
K \geq 1 + 2\max\crl*{1,\,\frac{10\delta}{\mu}}\prn*{1 + \ln\prn*{\frac{5B^2\mu (K-1)}{4\varepsilon}}}.
\end{equation}
Plugging the stepsize \eqref{eq:cor-sc-stepsize} into the second term of the guarantee from Theorem \ref{thm:strongly-convex} yields
\begin{equation*}
2\eta\sigma^2 = \frac{10\sigma^2}{\mu (K-1)}\prn*{1 + \ln\prn*{\frac{5B^2\mu (K-1)}{4\varepsilon}}}.
\end{equation*}
So, to make this smaller than $\varepsilon$, we require
\begin{equation*}
K \geq 1 + \frac{10\sigma^2}{\mu \varepsilon}\prn*{1 + \ln\prn*{\frac{5B^2\mu (K-1)}{4\varepsilon}}}.
\end{equation*}
Combining this with \eqref{eq:cor-first-term-K}, we conclude that with $\eta$ set according to \eqref{eq:cor-sc-stepsize}, if 
\begin{equation}\label{eq:cor-total-K-req}
K \geq 10\max\crl*{1,\,\frac{\delta}{\mu},\,\frac{\sigma^2}{\mu \varepsilon}}\prn*{1 + \ln\prn*{\frac{5B^2\mu K}{4\varepsilon}}}
\end{equation}
and, in addition $G \leq \sqrt{\mu\varepsilon}$, then the expected suboptimality is at most $3\varepsilon$, so using $\varepsilon' = \varepsilon / 3$ instead above ensures accuracy $\varepsilon$.

To conclude the proof, we note for $A,B > 0$, the inequality 
\begin{equation*}
X \geq A(1 + \ln(BX))
\end{equation*}
is satisfied by setting
\begin{equation*}
X \geq A(1 + 4\ln(e + AB)).
\end{equation*}
To show this, we first note that because $z\mapsto\sqrt{z}$ is a concave function on $z > 0$, we have the inequality
\begin{equation*}
\sqrt{z} \leq \sqrt{1} + (z-1)\prn*{\left.\frac{d}{dz}\sqrt{z}\right|_{z=1}} = 1 + \frac{z-1}{2} = \frac{1+z}{2}.
\end{equation*}
This implies that for $z > 0$ 
\begin{equation}
\frac{d}{dz}\sqrt{z} = \frac{1}{2\sqrt{z}} \geq \frac{1}{1+z} = \frac{d}{dz}\ln(1+z),
\end{equation}
Finally, since $\sqrt{0} = \ln(1+0)$, this implies that $\ln(1+z) \leq \sqrt{z}$ for all $z\geq 0$.

So, suppose that
\begin{equation}
X = \alpha A(1 + 4\ln(e + AB)).
\end{equation}
for some $\alpha \geq 1$. Then,
\begin{align*}
A(1+\ln(BX))
&= A(1 + \ln(\alpha) + \ln(AB) + \ln(1 + 4\ln(e + AB))) \\
&\leq A(1 + \ln(\alpha) + \ln(e+AB) + \sqrt{4\ln(e + AB)}) \\
&\leq A(1 + \ln(\alpha) + 3\ln(e+AB)) \\
&\leq \frac{1}{\alpha}X + \ln(c)A \\
&\leq \frac{1 + \ln(\alpha)}{\alpha}X \\
&\leq X,
\end{align*}
where we used that $1 + \ln(\alpha) \leq \alpha$ for all $\alpha \geq 1$.

Therefore, for a sufficiently large constant $c$, setting
\begin{equation*}
K = c\cdot\prn*{1 + \frac{\delta}{\mu} + \frac{\sigma^2}{\mu\varepsilon}}\ln\prn*{e + \prn*{1 + \frac{\delta}{\mu} + \frac{\sigma^2}{\mu\varepsilon}}\frac{\mu\E\nrm{\w_0 - \w^*}^2}{\varepsilon}}
\end{equation*}
satisfies \eqref{eq:cor-total-K-req}, completing the proof.
\end{proof}

\section{Proof of Theorem \ref{thm:non-convex}}\label{app:non-convex}

\begin{lemma}\label{lem:non-convex-function-decrease}
Under Assumptions \ref{ass:similarity} and \ref{ass:noise}, let $\eta \leq \frac{1}{4\delta}$. Then for any $\w$ such that $\E\phi_k(\w) \leq \E\phi_k(\w_k)$,
\[
\E\brk*{\f(\w) - \f(\w_k)}
\leq -\frac{1}{4\eta}\E\nrm*{\w - \w_k}^2 + 2\eta\sigma^2.
\]
\end{lemma}
\begin{proof}
Under Assumption \ref{ass:similarity}, $h = \f - \F$ is $\delta$-smooth. So,
\begin{equation*}
\f(\w) - \F(\w)
\leq \f(\w_k) - \F(\w_k) + \frac{\delta}{2}\nrm*{\w - \w_k}^2 
+ \inner{\nabla \f(\w_k) - \nabla \F(\w_k)}{\w - \w_k}   .
\end{equation*}
Rearranging and substituting $\phi_k$, this implies
\begin{equation*}
\f(\w) - \f(\w_k)
\leq \phi_k(\w) - \phi_k(\w_k) - \frac{1-\eta\delta}{2\eta}\nrm*{\w - \w_k}^2 
+ \inner{\nabla \f(\w_k) - \g_k}{\w - \w_k} .  
\end{equation*}
Next we use Young's inequality with Assumption \ref{ass:noise} to upper bound
\begin{equation*}
\E \inner{\nabla \f(\w_k) - \g_k}{\w - \w_k}  
\leq \frac{\eta\E\nrm{\nabla \f(\w_k) - \g_k}^2}{1-2\eta\delta} + \frac{1-2\eta\delta}{4\eta}\E\nrm{\w - \w_k}^2 \leq \frac{\eta\sigma^2}{1-2\eta\delta} + \frac{1-2\eta\delta}{4\eta}\E\nrm{\w - \w_k}^2 .
\end{equation*}
The fact that $\eta \leq \frac{1}{4\delta}$ and $\E\phi_k(\w) \leq \E\phi_k(\w_k)$ completes the proof.
\end{proof}

\begin{lemma}\label{lem:non-convex-distance-lower-bound}
Under Assumptions \ref{ass:similarity} and \ref{ass:noise}, let $\f$ be differentiable and let $\eta \leq \frac{1}{4\delta}$. If $\w_{k+1}$ satisfies
\[
\E\nrm{\phi_k(\w_{k+1})}^2 
\leq \frac{7}{16\eta^2}\E\nrm{\w_{k+1} - \w_k}^2 + \frac{1}{8}\nrm{\nabla \f(\w_{k+1})}^2
\]
then
\[
-\frac{1}{4\eta}\nrm{\w_{k+1} - \w_k}^2 
\leq -\frac{\eta}{48}\E\nrm*{\nabla \f(\w_{k+1})}^2 + \frac{\eta\sigma^2}{6}.
\]
\end{lemma}
\begin{proof}
By the relaxed triangle inequality, for any vectors $a,b,c,d$, 
\begin{equation}\label{eq:relaxed-triangle-inequality}
\nrm{a}^2 = \nrm{a+b+c+d-b-c-d}^2 
\leq 4(\nrm{a+b+c+d}^2 + \nrm{b}^2 + \nrm{c}^2 + \nrm{d}^2).
\end{equation}
Furthermore,
\begin{equation}
\nabla \phi_k(\w_{k+1}) = \nabla \F(\w_{k+1}) - \nabla \f(\w_{k+1}) - \nabla \F(\w_k) + \nabla \f(\w_k) 
+ \g_k - \nabla \f(\w_k) + \nabla \f(\w_{k+1}) + \frac{1}{\eta}(\w_{k+1} - \w_k).
\end{equation}
So \eqref{eq:relaxed-triangle-inequality} with 
\begin{align*}
a &= \nabla \f(\w_{k+1}) \\
b &= -\nabla \phi_k(\w_{k+1}) \\
c &= \nabla \F(\w_{k+1}) - \nabla \f(\w_{k+1}) - \nabla \F(\w_k) + \nabla \f(\w_k) \\
d &= \g_k - \nabla \f(\w_k) 
\end{align*}
so $a+b+c+d = \frac{1}{\eta}(\w_k - \w_{k+1})$, implies that under Assumptions \ref{ass:similarity} and \ref{ass:noise}
\begin{align*}
&\frac{1}{\eta^2}\nrm{\w - \w_k}^2\nonumber\\
&\geq 
\frac{1}{4}\E\nrm*{\nabla \f(\w)}^2 
- \E\nrm{\nabla \phi_k(\w)}^2 - \E\nrm{\g_k - \nabla \f(\w_k)}^2 - \E\nrm{\nabla \F(\w) - \nabla \f(\w) - \nabla \F(\w_k) + \nabla \f(\w_k)}^2 \\
&\geq 
\frac{1}{4}\E\nrm*{\nabla \f(\w)}^2 
- \E\nrm{\nabla \phi_k(\w)}^2 - \delta^2\E\nrm{\w-\w_k}^2 
- \sigma^2 \\
&\geq 
\frac{1}{8}\E\nrm*{\nabla \f(\w)}^2 
- \frac{1}{2\eta^2}\E\nrm{\w-\w_k}^2 
- \sigma^2,
\end{align*}
where we used the upper bound on $\E\nrm{\nabla \phi_k(\w)}^2$ and the fact that $\eta \leq \frac{1}{4\delta}$. Rearranging completes the proof.
\end{proof}

\nonconvex*
\begin{proof}
By Lemmas \ref{lem:non-convex-function-decrease} and \ref{lem:non-convex-distance-lower-bound}
\begin{equation}
\E\brk*{\f(\w_{k+1}) - \f(\w_k)}
\leq -\frac{1}{4\eta}\E\nrm*{\w_{k+1} - \w_k}^2 + 2\eta\sigma^2 
\leq -\frac{\eta}{48}\E\nrm*{\nabla \f(\w_{k+1})}^2 + \frac{\eta\sigma^2}{6}.
\end{equation}
Therefore, rearranging and averaging over $k$ yields
\begin{equation}
\frac{1}{K}\sum_{k=1}^K \E\nrm{\nabla \f(\w_k)}^2
\leq \frac{48\prn*{\f(\w_0) - \f(\w_K)}}{\eta K} + 8\sigma^2
\leq \frac{48\prn*{\f(\w_0) - \f^*}}{\eta K} + 8\sigma^2
\end{equation}
which completes the proof.
\end{proof}

\section{Proof of Proposition \ref{prop:making-gradient-small}}\label{app:proof-of-makinggradientsmall}
\makinggradientsmall*
\begin{proof}
Since $\phi_k(\w) = \F(\w) + \frac{1}{2\eta}\nrm{\w-\w_k}^2$ up to an additional affine term, it is easy to see that $\phi_k$ is $\frac{1}{\eta}$-strongly convex and $(H+\frac{1}{\eta})$-smooth. Therefore, for any $\w$
\begin{equation}
\nrm{\nabla \phi_k(\w)}^2 \leq 2\prn*{H + \frac{1}{\eta}}\prn{\phi_k(\w) - \phi_k(\w^*_{\phi})},
\end{equation}
where $\w^*_{\phi} = \argmin_\w \phi_k(\w)$.
In addition, the output, $\hat{\w}$, of \textsc{SGD} on $\phi_k$ using the optimal stepsize and initialized at $\w_k$ is, for universal constants $c,c' \geq 1$ \citep{nemirovskyyudin1983}
\begin{equation}
\E\brk{\phi_k(\hat{\w}) - \phi_k(\w^*_{\phi})} 
\leq c\prn*{H+\frac{1}{\eta}}\nrm{\w_k - \w^*_\phi}^2\exp\prn*{-\frac{T}{c'(1+H\eta)}} 
+ \frac{c\eta\rho^2}{T}.
\end{equation}
So, by the relaxed triangle inequality and the strong convexity of $\phi_k$
\begin{align}
\nrm*{\w_k - \w^*_\phi}^2
&\leq 2\nrm*{\hat{\w} - \w_k}^2 + 2\nrm*{\hat{\w} - \w^*_\phi}^2 \\
&\leq 2\nrm*{\hat{\w} - \w_k}^2 + 4\eta\prn*{\phi_k(\hat{\w}) - \phi_k(\w^*_\phi)} \\
&\leq 2\nrm*{\hat{\w} - \w_k}^2 + \frac{4c\eta^2\rho^2}{T} + 4c\prn*{1+H\eta}\exp\prn*{-\frac{T}{c'(1+H\eta)}}\nrm{\w_k - \w^*_\phi}^2.
\end{align}
Therefore, if 
\begin{equation}\label{eq:big-enough-T}
T \geq c'(1+H\eta)\ln\prn*{8c(1+H\eta)}
\end{equation}
then
\begin{equation}
\nrm*{\w_k - \w^*_\phi}^2 \leq 4\nrm*{\hat{\w} - \w_k}^2 + \frac{8c\eta^2\rho^2}{T},
\end{equation}
and thus
\begin{align}
\nrm{\nabla \phi_k(\hat{\w})}^2 
&\leq 2\prn*{H + \frac{1}{\eta}}\prn{\phi_k(\hat{\w}) - \phi_k(\w^*_{\phi})} \\
&\leq 2c\prn*{H+\frac{1}{\eta}}^2\exp\prn*{-\frac{T}{c'(1+H\eta)}}\nrm{\w_k - \w^*_\phi}^2 + \frac{2c(1+H\eta)\rho^2}{T} \\
&\leq 2c\prn*{H+\frac{1}{\eta}}^2\exp\prn*{-\frac{T}{c'(1+H\eta)}}\prn*{4\nrm{\hat{\w} - \w_k}^2 + \frac{8c\eta^2\rho^2}{T}} + \frac{2c(1+H\eta)\rho^2}{T} \\
&\leq 8c\prn*{H+\frac{1}{\eta}}^2\exp\prn*{-\frac{T}{c'(1+H\eta)}}\nrm{\hat{\w} - \w_k}^2 + \frac{4c(1+H\eta)\rho^2}{T} .
\end{align}
Therefore, choosing
\begin{equation}\label{eq:choice-of-T}
T \geq \max\bigg\{c'(1+H\eta)\ln\prn*{\frac{32c\prn*{1+H\eta}^2}{\mu\eta}}, 
\frac{4c(1+H\eta)\rho^2}{G^2} \bigg\}
\end{equation}
ensures that
\begin{equation}
\nrm{\nabla \phi_k(\hat{\w})}^2 \leq \frac{\mu}{4\eta}\nrm{\hat{\w} - \w_k}^2 + G^2.
\end{equation}
Finally, we note that our choice of $T$ \eqref{eq:choice-of-T} satisfies the condition \eqref{eq:big-enough-T} because $\F$ being $H$ smooth implies that $\f$ is $(H+\delta)$-smooth and therefore
\begin{equation*}
\mu \leq H+\delta \leq H + \frac{1}{\eta} \implies \frac{1+H\eta}{\eta\mu} \geq 1.
\end{equation*}
\end{proof}


\end{document}